\documentclass[11pt,doublespacing]{article}

\usepackage{fullpage}

\usepackage{setspace}   
\usepackage{graphicx}
\usepackage{amssymb}
\usepackage{natbib}
\usepackage{amsmath}
\usepackage{graphics}
\usepackage{bm}
\usepackage{color}
\usepackage{amsthm}
\usepackage{algorithm}
\usepackage[noend]{algpseudocode}
\usepackage{bbm}

\usepackage{enumerate}
\usepackage{multirow}
\usepackage{authblk}
\usepackage{verbatim}

\newcommand{\Var}{\mathrm{Var}}

\newtheorem{lem}{Lemma}
\newtheorem{prop}{Proposition}

\newtheorem{thm}{Theorem}
\newcommand{\bX}{{X}}
\newcommand{\bY}{{Y}}
\newcommand{\tbY}{\tilde{Y}}
\newcommand{\PtbY}{\Pi(\tilde{Y})}
\newcommand{\bZ}{{Z}}

\newcommand{\Oo}{\Omega_{{\rm obs}}}
\newcommand{\POo}{\mathcal{P}_{\Oo}}
\newcommand{\POoc}{\mathcal{P}_{\Oo^{\perp}}}
\newcommand{\PO}{\mathcal{P}_{\Omega}}
\newcommand{\POc}{\mathcal{P}_{\Omega^{\perp}}}
\newcommand{\PG}{\mathcal{P}_{\Gamma}}
\newcommand{\PGc}{\mathcal{P}_{\Gamma^{\perp}}}
\newcommand{\PT}{\mathcal{P}_{T}}
\newcommand{\PTp}{\mathcal{P}_{T^{\perp}}}
\newcommand{\PP}{\mathcal{P}_{\Psi}}
\newcommand{\PPp}{\mathcal{P}_{\Psi^{\perp}}}
\newcommand{\PA}{\POo}
\newcommand{\PAc}{\POoc}
\newcommand{\pz}{p_0}

\newcommand{\PQ}{\PGc Q}
\newcommand{\HP}{H^{\Psi}}
\newcommand{\HPp}{H^{\Psi^{\perp}}}
\newcommand{\MP}{M^{\Psi}}

\newcommand{\bigO}{\ensuremath{\mathop{}\mathopen{}\mathcal{O}\mathopen{}}}
\newcommand{\argmin}{\operatornamewithlimits{arg\ min}}

\newcommand{\softimpute}{{\sc Soft-Impute}}
\newcommand{\robustimpute}{{\sc Robust-Impute}}
\newcommand{\ignore}[1]{}

\begin{document}
\title{\LARGE\bf Matrix Completion with Noisy Entries and Outliers}
\author[1]{Raymond K. W. Wong\thanks{\noindent Department of Statistics, Texas A\&M University, College Station, TX 77843, U.S.A. Email: {\tt raywong@stat.tamu.edu}}}
\author[2]{Thomas C. M. Lee\thanks{\noindent Corresponding author.  Department of Statistics, University of California, Davis, One Shields Avenue, Davis, CA 95616, U.S.A. Email: {\tt tcmlee@ucdavis.edu}}}
\affil[1]{Department of Statistics, Texas A\&M University}
\affil[2]{Department of Statistics, University of California, Davis}
\date{September 25, 2017}

\maketitle

\begin{abstract}
This paper considers the problem of matrix completion when the observed entries are noisy and contain outliers.  It begins with introducing a new optimization criterion for which the recovered matrix is defined as its solution.  This criterion uses the celebrated Huber function from the robust statistics literature to downweigh the effects of outliers.  A practical algorithm is developed to solve the optimization involved.  This algorithm is fast, straightforward to implement, and monotonic convergent.  Furthermore, the proposed methodology is theoretically shown to be stable in a well defined sense.  Its promising empirical performance is demonstrated via a sequence of simulation experiments, including image inpainting.

\vspace*{0.5cm}

\noindent
{\it Keywords: ES-Algorithm, Huber function, robust methods, Soft-Impute, stable recovery}

\noindent
{\it Running title: Matrix Completion with Noises and Outliers}

\end{abstract}

\section{Introduction}
The goal of matrix completion is to impute those missing entries of a large matrix based on the knowledge of its relatively few observed entries.  It has many practical applications, ranging from collaborative filtering \citep{Rennie-Srebro05} to computer visions \citep{Weinberger-Saul06} to positioning \citep{Montanari-Oh10}.
In addition, its application to recommender systems is perhaps the most well known example, widely made popularized by the so-called Netflix prize problem \citep{Bennett-Lanning07}.  In this problem a large matrix of movie ratings is partially observed.  Each row of this matrix consists of ratings from a particular customer while each column records the ratings on a particular movie.  In the Netflix dataset, there are around $5\times10^5$ customers and $2\times 10^4$ movies, with less than $1\%$ of the ratings are observed.  Without any prior knowledge, a reasonable full recovery of the matrix is virtually impossible.  To overcome this issue, it is common to assume that the matrix is of low rank, reflecting the belief that the users' ratings are based on a relatively small number of factors.  This low rank assumption is very sensible in many applications, although the resulting optimizations are combinatorially hard \citep{Srebro-Jaakkola03}.  To this end, various convex relaxations and related optimization algorithms have been proposed to provide computationally feasible solutions; see, e.g., \citet{Candes-Recht09, Candes-Plan10, Keshavan-Montanari-Oh10a, Keshavan-Montanari-Oh10, Mazumder-Hastie-Tibshirani10, Marjanovic-Solo12} and \citet{Hastie-et-al14}.

In addition to computational advances, the theoretical properties of matrix completion using nuclear norm minimization have also been well studied.  For example, when the observed entries are noiseless, \citet{Candes-Recht09} show that perfect recovery of a low rank matrix is possible; see also \citet{Keshavan-Montanari-Oh10a}, \citet{Gross11} and \citet{Recht11}.  This result of \citet{Candes-Recht09} has been extended to noisy measurements by \citet{Candes-Plan10}: with high probability, the recovery is subject to an error bound proportional to the noise level.  Techniques that achieve this desirable property are often referred as {\em stable}.  See also \citet{Keshavan-Montanari-Oh10} and \citet{Koltchinskii-Lounici-Tsybakov11} for other theoretical developments of matrix completion from noisy measurements.

The original formulation of matrix completion 
assumes those observed entries are noiseless, and is later extended to the more realistic situation where the entries are observed with noise.
This paper further extend the formulation to simultaneously allow for both noisy entries and outliers.  To the authors knowledge, such an extension has not been considered before, although similar work exists.  In \citet{Candes-Li-Ma11} a method called principal component pursuit (PCP) is developed to recover a matrix observed with mostly noiseless entries and otherwise a small amount of outliers.  This is done by modeling the observed matrix as a sum of a low rank matrix and a sparse matrix.
\citet{Zhou-Li-Wright10} extend this PCP method to noisy entries but assumes the matrix is fully observed, thus it does not fall into the class of matrix completion problems.  Lastly \citet{Chen-Xu-Caramanis11} extend PCP to safeguard against special outlying structures, namely outlying columns.  However, it works only on outliers and otherwise noiseless entries.
Due to the similarity between the matrix completion and principal component analysis, it is worthmentioning that there are some related work \citep{Karhunen11, Luttinen-Ilin-Karhunen12} on robust principal component analysis with missing values.

The primary contribution of this paper is the development of a new robust matrix completion method that can be applied to recover a matrix with missing, noisy and/or outlying entries.  This method is shown to be stable in the sense of \citet{Candes-Plan10}, as discussed above.  As opposed to the above referenced PCP approach that decomposes the matrix into a sum of a low rank and a sparse matrix, the new approach is motivated by the statistical literature of robust estimation which modifies the least squares criterion to downweigh the effects of outliers.  Particularly, we make use of the Huber function for this modification.  We provide a theoretical result that establishes an intrinsic link between the two different approaches.  To cope with the nonlinearity introduced by the Huber function, we propose a fast, simple, and easy-to-implement algorithm to perform the resulting nonlinear optimization problem.  This algorithm is motivated by the ES-Algorithm for robust nonparametric smoothing \citep{Oh-Nychka-Lee07}.  As to be shown below, it can transform a rich class of (non-robust) matrix completion algorithms into algorithms for robust matrix completion.

The rest of this paper is organized as follows.  Section~\ref{sec:robust} provides further background of matrix completion and proposes a new optimization criterion for robust matrix recovery.  Fast algorithms are developed in Section~\ref{sec:alg} for practically computing the robust matrix estimate.  Theoretical and empirical properties of the proposed methodology are studied in Section~\ref{sec:thm} and Section~\ref{sec:sim} respectively.  Concluding remarks are given in Section~\ref{sec:con}, while technical details are relegated to the appendix.

\section{Matrix Completion with Noisy Observations and Outliers}\label{sec:robust}
Suppose $\bX$ is an $n_1\times n_2$ matrix which is observed for only a subset of entries ${\Oo}\subseteq [n_1]\times[n_2]$, where $[n]$ denotes $\{1,\dots,n\}$.  Let $\Oo^{\perp}$ be the complement of $\Oo$.  Define the projection operator $\mathcal{P}_{\Oo}$ as $\mathcal{P}_{\Oo} {B} = {C}$, where ${C}_{ij}={B}_{ij}$ if $(i,j)\in\Oo$ and ${C}_{ij}=0$ if $(i,j) \not\in \Oo$, for any $n_1\times n_2$ matrix $B=(B_{ij})_{i\in[n_1], j\in[n_2]}$.  The following is a standard formulation for matrix completion using a low rank assumption:
\begin{align*}
  \underset{{Y}}{\mbox{minimize}} \quad &\mathrm{rank}({Y})\\
\mbox{subject to}\quad
& \frac{1}{2}\|\mathcal{P}_{\Oo} {X} - \mathcal{P}_{\Oo} {Y}\|^2_F \le e, 
\end{align*}
where $e>0$ and $\|\cdot\|_F$ is the Frobenius norm.  Carrying out this rank minimization enables a good recovery of any low rank matrix with missing entries.  Note that for the reason of accommodating noisy measurements, the constraint above allows for a slight discrepancy between the recovered and the observed matrices.

However, this minimization is combinatorially hard \citep[e.g.,][]{Srebro-Jaakkola03}.  To achieve fast computation, the following convex relaxation is often used:
\begin{align*}
  \underset{{Y}}{\mbox{minimize}} \quad &\|{Y}\|_{*}\\
\mbox{subject to}\quad
& \frac{1}{2}\|\mathcal{P}_{\Oo} {X} - \mathcal{P}_{\Oo} {Y}\|^2_F \le e,
\end{align*}
where $\| {Y} \|_{*}$ represents the nuclear norm of ${Y}$ (i.e., the sum of singular values of ${Y}$). The Lagrangian form of this optimization is
\begin{equation}
  \underset{{Y}}{\mbox{minimize}} \quad
 f(\bY|\bX)\equiv\frac{1}{2}\|\mathcal{P}_{\Oo} {X} - \mathcal{P}_{\Oo} {Y}\|^2_F +
 \gamma\|{Y}\|_{*},\label{eqn:nonrobust}
\end{equation}
where $\gamma >0 $ has a one-to-one correspondence to $e$.  The squared loss in the first term is used to measure the fitness of the recovered matrix to the observed matrix.  It is widely known that such a squared loss is very sensitive to outliers and often leads to unsatisfactory recovery results if such outliers exist.  Motivated by the literature of robust statistics \citep[e.g.,][]{Huber-Ronchetti11}, we propose replacing this squared loss by the Huber loss function
\[
  \rho_c(x) = 
  \begin{cases}
    x^2, & |x| \le c\\
    c(2|x| - c), & |x|>c
\end{cases},
\]
with tuning parameter $c$.
When comparing with the squared loss, the Huber loss downweighs the effects of extreme measurements.  Our proposed solution for robust matrix completion is given by the following minimization:
\begin{equation}
  \underset{{Y}}{\mbox{minimize}} \quad
g(\bY)\equiv\frac{1}{2}\sum_{(i,j)\in\Oo} \rho_c( X_{ij}-Y_{ij} )
 +
 \gamma\|{Y}\|_{*}.\label{eqn:goal}
\end{equation}
Note that the convexity of $\rho_c$ guarantees the convexity of the objective criterion~(\ref{eqn:goal}).

For many robust statistical estimation problems the tuning parameter $c$ is pre-set as $c=1.345\hat{\sigma}$ to achieve a 95\% statistical efficiency, where $\hat{\sigma}$ is an estimate of the standard deviation of the noise.  For the current problem, however, the choice of $c$ is suggested by Theorem~\ref{thm:stable} below: $c=\gamma/\sqrt{n_{(1)}p}$, where $n_{(1)}=\max\{n_1, n_2\}$ and $p$ is the percentage of missing entries.  This choice of $c$ was used throughout all our numerical work.

\section{Fast Algorithms for Minimization of~(\protect\ref{eqn:goal})}
\label{sec:alg}

Since the gradient of the Huber function is non-linear, (\ref{eqn:goal}) is a harder optimization problem when comparing to many typical matrix completion formulations such as~(\ref{eqn:nonrobust}).  As an example, consider~(\ref{eqn:nonrobust}) when $\bX$ is fully observed; i.e., $\Oo=[n_1]\times[n_2]$.  Through sub-gradient analysis \citep[e.g.,][]{Cai-Candes-Shen10, Ma-Goldfarb-Chen11}, one can derive a closed-form solution to (\ref{eqn:nonrobust}), denoted as $S_\gamma(X)$, where $S_\gamma$ is the soft-thresholding operator defined in \citet{Mazumder-Hastie-Tibshirani10}, also given in~(\ref{eqn:Sgamma}) below.  However, even if $X$ was fully observed, (\ref{eqn:goal}) does not have a closed-form solution.  The goal of this section is to develop fast methods for minimizing~(\ref{eqn:goal}).

\subsection{A General Algorithm}
In \citet{Oh-Nychka-Lee07} a method based on the so-called theoretical construct {\em pseudo data} is proposed for robust wavelet regression.  The idea is to transform a Huber-type minimization problem into a sequence of fast and well understood squared loss minimization problems.  This subsection modifies this idea and proposes an algorithm to minimizing~(\ref{eqn:goal}).

As similar to \citet{Oh-Nychka-Lee07}, we define a {\em pseudo data matrix} as
\begin{equation}
{Z} = \PA \tilde{Y} + \frac{1}{2} \psi_{c}({E}),
\label{eqn:pseudomatrix}
\end{equation}
where $\tilde{Y}$ is the current estimate of the target matrix, ${E} = \PA {X} - \PA \tilde{Y}$ is the ``residual matrix'', and $\psi_c=\rho_c'$ is the derivative of $\rho_c$. With a slight notation abuse, when $\psi_c$ is applied to a matrix, it means $\psi_c$ is evaluated in an element-wise fashion.  Straightforward algebra shows that the sub-gradient of $f(Y|Z)$ (with respect to $Y$) evaluated at $\tilde{Y}$,
\begin{equation}
-(\POo Z -\POo \tilde{Y}) + \gamma\partial \| \tilde{Y}\|_*, \label{eqn:sub1}
\end{equation}
is equivalent to the sub-gradient of $g(Y)$ (with respect to $Y$) evaluated at $\tilde{Y}$,
\begin{equation}
  -\frac{1}{2}\psi_c(\POo X -\POo \tilde{Y}) + \gamma \partial \| \tilde{Y}\|_*.\label{eqn:sub2}
\end{equation}
The proposed algorithm iteratively updates $\tilde{Y}=\argmin_Y f(Y|Z)$ and $Z$ using~(\ref{eqn:pseudomatrix}).  Upon convergence (implied by Proposition~\ref{prop:mono} below), the sub-gradient~(\ref{eqn:sub1}) contains 0 at the converged $\tilde{Y}$ and thus the sub-gradient~(\ref{eqn:sub2}) also contains 0 at this converged $\tilde{Y}$.  Therefore this $\tilde{Y}$ is the solution to~(\ref{eqn:goal}).  Details of this algorithm based on pseudo data matrix are given in Algorithm~\ref{alg:general}.

\begin{algorithm}
  \caption{The General Robust Algorithm}\label{alg:general}
  \begin{algorithmic}[1]
    \State Perform (non-robust) matrix completion on $X$ and assign
    $Y^{{\rm old}} \leftarrow \argmin_Y f(Y|X)$.  This $Y^{{\rm old}}$ is the initial estimate
    (starting point of the algorithm).
    \State Repeat:
      \begin{enumerate}[(a)]
        \item Compute ${E} \leftarrow \PA {X} - \PA {Y}^{{\rm old}}$.
        \item Compute ${Z} \leftarrow \PA {Y}^{{\rm old}} +
            \frac{1}{2} \psi_{c}({E})$.
          \item Perform (non-robust) matrix completion on $Z$ and assign
            ${Y^{{\rm new}}} \leftarrow \argmin_Y f(Y|Z)$.
\item If
  \[
    \frac{\|{Y}^{{\rm new}} - {Y}^{{\rm old}}\|_F^2}{\|{Y}^{{\rm old}}\|_F^2} < \varepsilon,
  \]
  exit.
\item Assign ${Y}^{{\rm old}} \leftarrow {Y}^{{\rm new}}$.
  \end{enumerate}
  \State Output $Y^{{\rm new}}$.
\end{algorithmic}
\end{algorithm}

Algorithm~\ref{alg:general} has several attractive properties.  First, it can be paired with any existing (non-robust) matrix completion algorithm (or software), as can be easily seen in Step~2(c).  This is a huge advantage, as a rich body of existing (non-robust) methods can be made robust against outliers.
Second, once such an (non-robust) algorithm is available, the rest of the implementation is straightforward and simple, and no expensive matrix operations are required.  Lastly, it has strong theoretical backup, as to be reported in Section~\ref{sec:thm}.

\subsection{Further Integration with Existing Matrix Completion Algorithms}
Many existing matrix completion algorithms are iterative.  A direct application of Algorithm~\ref{alg:general} would lead to an algorithm that is iterations-within-iterations.  Although our extensive numerical experience suggests that these direct implementations would typically converge within a few iterations to give a reasonably fast execution time, it would still be advantageous to speed up the overall procedure.  Here we show that it is possible to further improve the speed of the overall robust algorithm by embedding the pseudo data matrix idea directly into a non-robust algorithm.

We shall illustrate this with the \softimpute\ algorithm proposed by \citet{Mazumder-Hastie-Tibshirani10}.  To proceed we first recall the definition of their thresholding operator $S_\gamma$: for any matrix ${Z}$ of rank $r$, 
\begin{equation}
S_\gamma({Z})={U} {D}_\gamma {V}^\intercal, 
\label{eqn:Sgamma}
\end{equation}
where ${Z}={U}{D}{V}^\intercal$ is the singular value decomposition of ${Z}$, ${D} = \mathrm{diag}[d_1,\dots,d_r]$ and ${D}_{\gamma} = \mathrm{diag}[(d_1-\gamma)_{+},\dots,(d_r-\gamma)_{+}]$.  Now the main idea is to suitably replace an iterative matrix estimate with the pseudo data matrix estimate given by~(\ref{eqn:pseudomatrix}).  With \softimpute, the resulting robust algorithm is given in Algorithm~\ref{alg:essoft}.  We shall call this algorithm \robustimpute.  As to be shown by the numerical studies below, \robustimpute\ is very fast and produces very promising empirical results.
Our algorithm also has the sparse-plus-low-rank structure in the singular
  value thresholding step (Step 2a(iii)).
  This linear algebra structure has positive impact on the computational complexity.
  See Section 5 of \citet{Mazumder-Hastie-Tibshirani10} for details.
Moreover, the monotonicity and convergence of our algorithm
is guaranteed by Proposition~\ref{prop:mono}
and Theorem~\ref{thm:global}.

\begin{algorithm}
  \caption{\robustimpute}\label{alg:essoft}
  \begin{algorithmic}[1]
    \State Initialize ${{Y}}^{{\rm old}}=S_{\gamma_1}(\PA {X})$ and ${Z}={X}$.
  \State Do for $\gamma_1>\gamma_2>\dots > \gamma_K$:
  \begin{enumerate}[(a)]
    \item Repeat:
      \begin{enumerate}[(i)]
        \item Compute ${E} \leftarrow \PA {X} - \PA {Y}^{{\rm old}}$.
        \item Compute ${Z} \leftarrow \PA {Y}^{{\rm old}} +
            \frac{1}{2} \psi_{c}({E})$
        \item Compute ${Y}^{{\rm new}} \leftarrow S_{\gamma_k}(\PA {Z} +
  \PAc{Y}^{{\rm old}})$.
\item If
  \[
    \frac{\|{Y}^{{\rm new}} - {Y}^{{\rm old}}\|_F^2}{\|{Y}^{{\rm old}}\|_F^2} < \varepsilon,
  \]
  exit.
\item Assign ${Y}^{{\rm old}} \leftarrow {Y}^{{\rm new}}$.
  \end{enumerate}
\item Assign $\hat{{Y}}_{\gamma_k} \leftarrow {Y}^{{\rm new}}$.
  \end{enumerate}
  \State Output the sequence of solutions $\hat{{Y}}_{\gamma_1},\dots,\hat{{Y}}_{\gamma_K}$.
\end{algorithmic}
\end{algorithm}

\section{Theoretical Properties}\label{sec:thm}
This section presents some theoretical backups for the proposed methodology.

\subsection{Monotonicity and global convergence}
We first present the following proposition concerning the monotonicity of the algorithms.  The proof can be found in~Appendix~\ref{app:proofp1}.
We also provide an alternative proof suggested by a referee, based on the idea of alternating minimization, in Appendix~\ref{app:proofp1}

\begin{prop}[Monotonicity]\label{prop:mono}
Let ${{Y}}^{(k)}$ and ${Z}^{(k)}=\POo Y^{(k-1)} + \psi_c(\POo X-\POo Y^{(k-1)})/2$ be, respectively, the estimate and the pseudo data matrix in the $k$-th iteration. If $Y^{(k+1)}$ is the next estimate such that $f(\bY^{(k+1)}|\bZ^{(k+1)}) \le f(\bY^{(k)}|\bZ^{(k+1)})$, then $g(\bY^{(k+1)}) \le g(\bY^{(k)})$.
\end{prop}

For the general version (Algorithm~\ref{alg:general}), it is obvious that the condition $f(\bY^{(k+1)}|\bZ^{(k+1)}) \le f(\bY^{(k)}|\bZ^{(k+1)})$ is satisfied as the result of the minimization $Y^{{\rm old}} \leftarrow \argmin_Y f(Y|Z)$.  For the specialized version \robustimpute\ (Algorithm~\ref{alg:essoft}), this condition is implied by Lemma~2 of \citet{Mazumder-Hastie-Tibshirani10}.
Therefore both versions are monotonic.

As pointed out by a referee, 
  the proposed algorithms
  can also be viewed as an instance of the majorization-minimization (MM)
  algorithm \citep{Lange-Hunter-Yang00, Hunter-Lange04}.
  It can be shown that, for $(i,j)\in\Oo$,
  \begin{align*}
    \rho_c(X_{ij}-Y_{ij}) &\le \rho_c(X_{ij}-Y_{ij}^{\rm old}) - (Y_{ij}- Y_{ij}^{\rm old})\psi_c(X_{ij}-Y_{ij}^{\rm old}) + 2 \cdot \frac{1}{2}(Y_{ij}-Y_{ij}^{\rm old})^2\\
    &= \left[ Y_{ij} - Y_{ij}^{\rm old} -\frac{1}{2}\psi_c(X_{ij}-Y_{ij}^{\rm old}) \right]^2 + constant\\
    &= ( Y_{ij} - Z_{ij} )^2 + constant.
  \end{align*}
  Therefore,
  subject to an additive constant that does not depend on
  $Y$,
  $h(Y|Y^{\rm old})=f(Y|Z)=(1/2)\sum_{(i,j)\in\Oo} (Z_{ij}-Y_{ij})^2 + \gamma \|Y\|_*$ is a
  majorization of the objective function $g$.
  With this majorization, Algorithm \ref{alg:general} can be viewed as 
  an MM algorithm.
  Additionally, one can majorize the unobserved entries
  by $(Y_{ij}-Z_{ij})^2=(Y_{ij}-Y_{ij}^{\rm old})^2\ge 0$
  and, together with the above majorization of the observed entries,
 Algorithm \ref{alg:essoft} can also be shown as
  an MM algorithm.
  Therefore the monotonicity of the proposed algorithms
  can also be obtained by the general theory of MM algorithm
  \citep[e.g.,][]{Lange10}.
  Moreover, the explicit connection to the MM algorithm allows
  possible extensions of the current algorithm to other robust loss
  functions such as Tukey's biweight loss.
  However, due to non-differentiability of the objective function,
  the typical convergence analysis of MM algorithm
  \citep[e.g.,][Ch.~15]{Lange10} does not apply to our case.

  We summarize the global convergence rates of both Algorithm \ref{alg:general} and
  Algorithm \ref{alg:essoft} in the following theorem.
   \begin{thm} \label{thm:global}
     Let $\bY^{(k)}$ and $\bY^{(0)}$ be, respectively, the estimate in the $k$-th iteration and
     the starting point of Algorithm \ref{alg:general} or Algorithm \ref{alg:essoft}
     Then for any $k\ge 1$,
     \begin{align*}
       \mbox{Algorithm \ref{alg:general}:} \qquad g(\bY^{(k)}) - g(\bY^{*}) &\le \frac{\| \POo\bY^{(0)} - \POo\bY^{*}\|^2_F}{2k}, \qquad
       \forall Y^*\in \mathcal{Y},\\
       \mbox{Algorithm \ref{alg:essoft}:} \qquad g(\bY^{(k)}) - g(\bY^{*}) &\le \frac{\| \bY^{(0)} - \bY^{*}\|^2_F}{2k}, \qquad
       \forall Y^*\in \mathcal{Y},
     \end{align*}
     where $\mathcal{Y}$ be the set of all global
     minimizers of $g$ (i.e. $\mathcal{Y}=\arg\min_{\bY\in\mathbb{R}^{n_1\times
         n_2}} g(\bY)$).
   \end{thm}
The global convergence analysis of Algorithm \ref{alg:general} can be carried out
similarly as in \citet{Beck-Teboulle09} for proximal gradient method,
despite that Algorithm \ref{alg:general} is not a proximal gradient method.
For completeness, we give the proof of Theorem \ref{thm:global} for Algorithm \ref{alg:general}
in Appendix \ref{app:proofalg}.

  As for \robustimpute
  \ (Algorithm \ref{alg:essoft}),
  we can rewrite it as an instance of the proximal gradient method
  applied to $g(\bY)= g_1(\bY_1) + g_2(\bY_2)$,
  where $g_1(\bY) = (1/2) \sum_{(i,j)\in\Oo} \rho_c(X_{ij} - Y_{ij})$
 and $g_2(\bY)= \gamma \|Y\|_*$.
 In our case, the proximal gradient method with step size $L$ iterates over
 $\bY^{(k+1)}=\xi_L(\bY^{(k)})$ with
 \[
   \xi_L (\tilde{\bY}) =
   \arg\min_{\bY} \left\{ g_2(\bY) +
     \frac{L}{2} \left\|  \bY - \left(  \tilde{\bY} - \frac{1}{L} \nabla
         g_1(\tilde{\bY}) \right) \right\|_F^2
   \right\},
 \]
 where $L$ is constant greater than or equal to the Lipstchiz constant of $g_1$.
 Note that $g_1$ has a Lipschitz contant 1.
 If we take $L=1$, we have the following simplification.
 \begin{align*}
     g_2(\bY) +
     \frac{L}{2} \left\|  \bY - \left(  \tilde{\bY} - \frac{1}{L} \nabla
         g_1(\tilde{\bY}) \right) \right\|_F^2
     &= g_2(\bY) + \frac{1}{2}\left\|
       \bY - \left\{ \tilde{Y} + \frac{1}{2} \psi_c (\POo \bX - \POo
         \tilde{\bY}) \right\} \right\|_F^2\\
     &= g_2(\bY) + \frac{1}{2}\left\| 
       \bY - \left\{ \POoc \tilde{Y} + \POo \bZ \right\} \right\|_F^2.
   \end{align*}
   The minimization of ${\xi}_1$ is equivalent to Step 2a(iii)
   of Algorithm \ref{alg:essoft}.
   Therefore, the proximal gradient method is the same as \robustimpute.
   This connection allows us to apply the convergence results of proximal gradient
   method to \robustimpute \ directly.
   Theorem \ref{thm:global} for Algorithm \ref{alg:essoft} follows from Theorem
   3.1~of \cite{Beck-Teboulle09}.
   Lastly, the Nesterov's method \citep{Nesterov07} can be applied directly to
    accelerate Algorithm 2.
    The resulted accelerated version is expected to be faster in terms of convergence.
    However, the acceleration in the Nesterov's method ruins
    the computationally beneficial
     sparse-plus-low-rank structure \citep{Mazumder-Hastie-Tibshirani10} in the singular vaue thresholding step (Step 2a(iii)).
    Hence, for large matrices, the non-accelerated version is still preferred in terms of overall computations. The detailed discussion can be found in Section 5 of \citet{Mazumder-Hastie-Tibshirani10}.

\subsection{Stable Recovery}
Recall the stable property of \citet{Candes-Plan10} implies that, with high probability, the recovered matrix is subject to an error bound proportional to the noise level.  This subsection shows that the robust matrix completion defined by~(\ref{eqn:goal}) is also stable. 

Although the formulation of~(\ref{eqn:goal}) has its root from classical robust statistics, it is also related to the more recent principal component pursuit (PCP) proposed by \citet{Candes-Li-Ma11}.  PCP assumes that the entries of the observed matrix are noiseless, and that this matrix can be decomposed as the sum of a low rank matrix and a sparse matrix, where the sparse matrix is treated as the gross error.  In \citet{Candes-Li-Ma11} it is shown that using PCP perfect recovery is possible with or without missing entries in the observed matrix.
Another notable work by \citet{Chandrasekaran-Sanghavi-Parrilo11} provide completely deterministic conditions for the PCP to succeed under no missing data.
  See Section 1.5 of \citet{Candes-Li-Ma11} for a detailed comparison between these two pieces of work.
For the case of noisy measurements without missing entries, \citet{Zhou-Li-Wright10} extend PCP to stable PCP (SPCP), which is shown to be stable.  However, to the best of our knowledge, there is no existing theoretical results for the case of noisy (and/or outlying) measurements with missing entries.

Inspired by \citet{She-Owen11}, we first establish an useful link between robust matrix completion~(\ref{eqn:goal}) and PCP in the following proposition.  The proof can be found in Appendix~\ref{app:proofp2}.

\begin{prop}[Equivalence]\label{prop:equiv}
The minimization (\ref{eqn:goal}) is equivalent to
\begin{equation}
  \underset{L, S}{\mathrm{minimize}}\quad\frac{1}{2} \| \PA \bX - \PA
  (L + S)\|_F^2 + \gamma\|L\|_* + c \|S\|_1.
  \label{eqn:goal2}
\end{equation}
That is, the minimizing $\bY$ of~(\ref{eqn:goal}) and the minimizing $L$ of~(\ref{eqn:goal2})
coincide.
\end{prop}

Minimization~(\ref{eqn:goal2}) has a high degree of similarity to both PCP and SPCP.  It is equivalent to
\begin{align}
  \underset{{L,S}}{\mbox{minimize}} \quad &\|L\|_* + \lambda \|S\|_1 \label{eqn:goal3}\\
\mbox{subject to}\quad
& \|\mathcal{P}_{\Oo} {X} -
\mathcal{P}_{\Oo} (L+S)\|^2_F \le \delta^2, \nonumber
\end{align}
where $\lambda=c/\gamma$ and $\delta>0$ has a one-to-one correspondence to $\gamma$.  When comparing with PCP, (\ref{eqn:goal2}) permits the observed matrix to be different from the recovered matrix ($L+S$) to allow for noisy measurements.  When comparing with SPCP, (\ref{eqn:goal2}) permits missing entries, which is necessary for matrix completion problems.

Proposition~\ref{prop:equiv} has two immediate implications.  First, the proposed Algorithm~\ref{alg:general} provides a general methodology to turn a large and well-developed class of matrix completion algorithms into algorithms for solving SPCP with missing entries.  Second, many useful results from PCP can be borrowed to study the theoretical properties of robust matrix completion~(\ref{eqn:goal}).  In particular, we show that~(\ref{eqn:goal}) leads to stable recovery.  With Proposition~\ref{prop:equiv}, it suffices to show that~(\ref{eqn:goal2}) achieves stable recovery of $(L_0, S'_0)$ from the data $\PA(X)$ generated by $\PA(L_0+S_0)$ obeying $\|\PA X-\PA(L_0+S_0)\|_F\le\delta$ and $S_0'=\PA S_0$.  Note that $L_0=X_0$.

We need some notations to proceed.  For simplicity, we assume $n=n_1=n_2$ but our results can be easily extended to rectangular matrices ($n_1\neq n_2$).  The Euclidean inner product $\langle Q,R\rangle$ is defined as $\mathrm{trace}(Q^\intercal R)$.  Let $\pz$ be the proportion of observed entries.  Write $\Gamma\subset \Oo$ as the set of locations where the measurements are noisy (but not outliters), and $\Omega=\Oo \backslash \Gamma$ as the support of $S_0'=\POo S_0$; i.e., locations of outliers.  Denote their complements as, respectively, $\Gamma^{\perp}$ and $\Omega^{\perp}$.  We define $\PG$, $\PO$, $\PGc$ and $\POc$ similarly to the definition of $\POo$.  Let $r$ be the rank of $L_0$ and $UDV^{\intercal}$ be the corresponding singular value decomposition of $L_0$, where $U, V\in \mathbb{R}^{n\times r}$ and $D\in\mathbb{R}^{r\times r}$.  Similar to \citet{Candes-Li-Ma11}, we consider the linear space of matrices
\[
  T :=\{U Q^\intercal + R V^{\intercal} : Q, R \in \mathbb{R}^{n\times r}\}.
\]
Write $\PT$ and $\PTp$ as the projection operator to $T$ and $T^\perp$ respectively.  As in \citet{Zhou-Li-Wright10}, we define a set of notations for any pair of matrices $M=(L,S)$. Here, let $\|M\|_F :=\sqrt{\|L\|^2_F + \|S\|_F^2}$ and $\|M\|_\Diamond := \|L\|_* + \lambda\|S\|_1$.  We also define the projection operators $\PT \times \PGc : (L,S) \mapsto (\PT L, \PGc S)$ and $\PTp \times \PG : (L,S) \mapsto (\PTp L, \PG S)$.  In our theoretical development, we consider the following special subspaces
\begin{align*}
  \Psi &:= \{ (L,S) :  L, S\in\mathbb{R}^{n\times n}, \POo L = \POo S,
\POoc L=\POoc S=0\},\\
  \Psi^\perp &:= \{ (L,S) :  L, S\in\mathbb{R}^{n\times n}, \POo L + \POo S =0\}.
\end{align*}
And we write the corresponding projection operators as $\PP$ and $\PPp$ respectively.  Let $M_0=(L_0, S_0')$.  Lastly, for any linear operator $\mathcal{A}$, the operator norm, denoted by $\|\mathcal{A}\|$, is $\sup_{\{\|Q\|_F=1\}} \|\mathcal{A} Q\|_F$.  In below, we write that an event occurs with high probability if it holds with probability at least $1-\bigO(n^{-10})$.

To avoid certain pathological cases \citep[see, e.g.,][]{Candes-Recht09},
an incoherence condition on $U$ and $V$ is usually assumed.  To be specific, this condition with the parameter $\mu$ is:
\begin{equation}
  \max_i \| U^\intercal e_i\|^2 \le \frac{\mu r}{n_1},
  \quad
  \max_i \| V^\intercal e_i\|^2 \le \frac{\mu r}{n_2},
  \quad
  \mbox{and}
  \quad
  \|UV^\intercal\|_\infty \le \sqrt{\frac{\mu r}{n_1 n_2}},
  \label{eqn:incoherence}
\end{equation}
where $\|Q\|$ is the operator norm or 2-norm of matrix $Q$ (i.e., the largest singular value of $Q$) and $\|Q\|_\infty=\max_{i,j}|Q_{i.j}|$.  This condition guarantees that, for small $\mu$, the singular vectors are reasonably spread out.

\begin{thm}[Stable Recovery]\label{thm:stable}
Suppose that $L_0$ obeys~(\ref{eqn:incoherence}) and $\Oo$ is uniformly distributed among all sets of cardinality $m=p_0n^2$ with $p_0>0$ being the proportion of observed entries.  Further suppose that each observed entry is grossly corrupted to be an outlier with probability $\tau$ independently of the others.  Suppose $L_0$ and $S_0$ satisfy $r \le \rho_r n \mu^{-1} (\log n)^{-2}$ and $ \tau \le \tau_s$ with $\rho_r, \tau_s$ being positive numerical constants. Choose $\lambda=1/\sqrt{np_0}$. Then, with high probability (over the choices of $\Omega$ and $\Oo$), for any $X$ obeying $\|\PA X-\PA(L_0+S_0)\|_F\le\delta$, the solution $(\hat{L}, \hat{S})$ to (\ref{eqn:goal3}) satisfies
\[
    \|\hat{L}-L_0\|_F \le \left\{2 + 8\sqrt{n}\left(1+\sqrt{\frac{8}{{\pz}}}\right) \right\} \delta
    \qquad\mbox{and}\qquad 
    \|\hat{S}-S_0'\|_F \le \left\{2 + 8\sqrt{n}\left(1+\sqrt{\frac{8}{{\pz}}}\right) \right\}\sqrt{np_0} \delta,
  \]
where $S_0'=\PA(S_0)$.
\end{thm}
The proof of this theorem can be found in Appendix~\ref{app:prooft1}.

\section{Empirical Performances}\label{sec:sim}
Two sets of numerical experiments and a real data application were conducted to evaluate the practical performances of the proposed methodology.  In particular the performance of the proposed procedure \robustimpute\ is compared to the performance of \softimpute\ developed by \citet{Mazumder-Hastie-Tibshirani10}.  The reasons \softimpute\ is selected for comparison are that it is one of the most popular matrix completion methods due to its simplicity and scalability, and that it is shown by \citet{Mazumder-Hastie-Tibshirani10} that it generally produces superior results to other common matrix completion methods such as {\sc MMMF} of \citet{Rennie-Srebro05}, {\sc SVT} of \citet{Cai-Candes-Shen10} and {\sc OptSpace} of \citet{Keshavan-Montanari-Oh10a}

\subsection{Experiment 1: Gaussian Entries}
This experiment covers those settings used in \citet[][Section~9]{Mazumder-Hastie-Tibshirani10} and additional settings with different proportions of missing entries and outliers.  For each simulated data set, the target matrix was generated as $X_0 = U V^\intercal$, where $U$ and $V$ are random matrices of size $100 \times r$ with independent standard normal Gaussian entries.  Then each entry of $X_0$ is contaminated by additional independent Gaussian noise with standard deviation $\sigma$, which is set to a value such that the signal-to-noise ratio (SNR) is 1.  Here SNR is defined as
\[
  \mbox{SNR} = s = \sqrt{\frac{\Var(X_0)}{\sigma^2}},
\]
where $\Var(X_0)$ is the variance over all the entries of $X_0$ conditional on $U$ and $V$.  Next, for each entry, with probability $p$ yet another independent Gaussian noise with $\sigma/4$ is added; these entries are treated as outliers.  We call this contaminated version of $X_0$ as $X$.  Lastly, $\Oo$ is uniformly random over the indices of the matrix with missing proportion as $q$.  In this study, we used two values for $r$ (5, 10), three values for $p$ (0, 0.05, 0.1) and three values for $q$ (0.25, 0.5, 0.75).  Thus in total we have 18 simulation settings.  For each setting 200 simulated data sets were generated, and both the non-robust method \softimpute\ and the proposed \robustimpute\ were applied to recover $X_0$.
We also provide two oracle fittings as references. They are produced by
applying \softimpute\ to the simulated data set with outlying observed entries removed (i.e.,
treated as missing entries),
and with outlying observed entries replaced by non-outlying contaminated entries (i.e., contaminated by independent Gaussian noise with standard deivation $\sigma$) respectively.
The first oracle fitting is referred to as oracle1 while the second one is called oracle2 in the following.

For the two simulation settings with $r=10$ and $q=0.5$, and one with $p=0$ while the other with $p=0.1$, Figure~\ref{fig:sim} summarizes the average number of singular value decompositions (SVDs) used and the average test error.  Here test error is defined as
\[
  \mbox{Test error} = \frac{\| \POoc(X_0 - \hat{X}) \|_F^2}{\| \POoc X_0 \|_F^2},
\]
where $\PG$ is the projection operator to the set of locations of the observed noisy entries (but not outliers) $\Gamma$, and $\hat{X}$ is an estimate of $X_0$.
From Figure~\ref{fig:sim}~(Top), one can see that the performance of \robustimpute\ is slightly inferior to \softimpute\ in the case of no outliers ($p=0$), while \robustimpute\ gave significantly better results when outliers were present ($p=0.1$).  The inferior performance of \robustimpute\ under the absence of outliers is not surprising, as it is widely known in the statistical literature that a small fraction of statistical efficiency would be lost when a robust method is applied to a data set without outliers.  However, it is also known that the gain could be substantial if outliers did present.

As for computational requirements, one can see from Figure~\ref{fig:sim} (Bottom) that \robustimpute\ only used slightly more SVDs on average.  For ranks greater than 5, the number of SVDs used by \robustimpute\ only differs from \softimpute\ on average by less than 1.  This suggests that \robustimpute\ is slightly more computationally demanding than \softimpute.

Similar experimental results were obtained for the remaining 16 simulation settings.  For brevity, the corresponding results are omitted here but can be found in the supplementary document.

From this experiment some empirical conclusions can be drawn.  When there is no outlier, \softimpute\ gives slightly better results, while with outliers, results from \robustimpute\ are substantially better.  Since that in practice one often does not know if outliers are present or not, and that \robustimpute\ is not much more computationally demanding than \softimpute, it seems that \robustimpute\ is the choice of method if one wants to be more conservative.

\begin{figure}[htpb]
  \includegraphics[height=9cm]{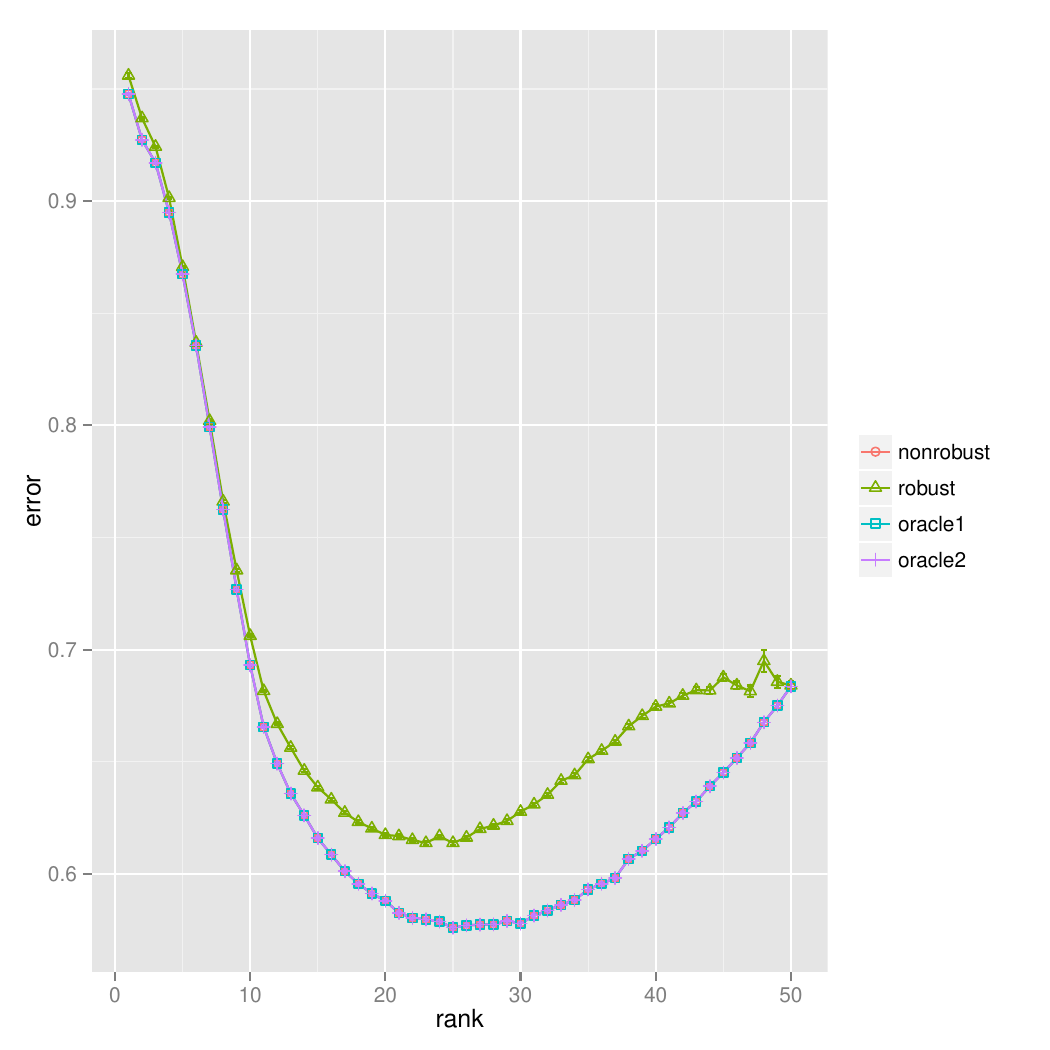}
  \hspace{-0.5cm}
  \includegraphics[height=9cm]{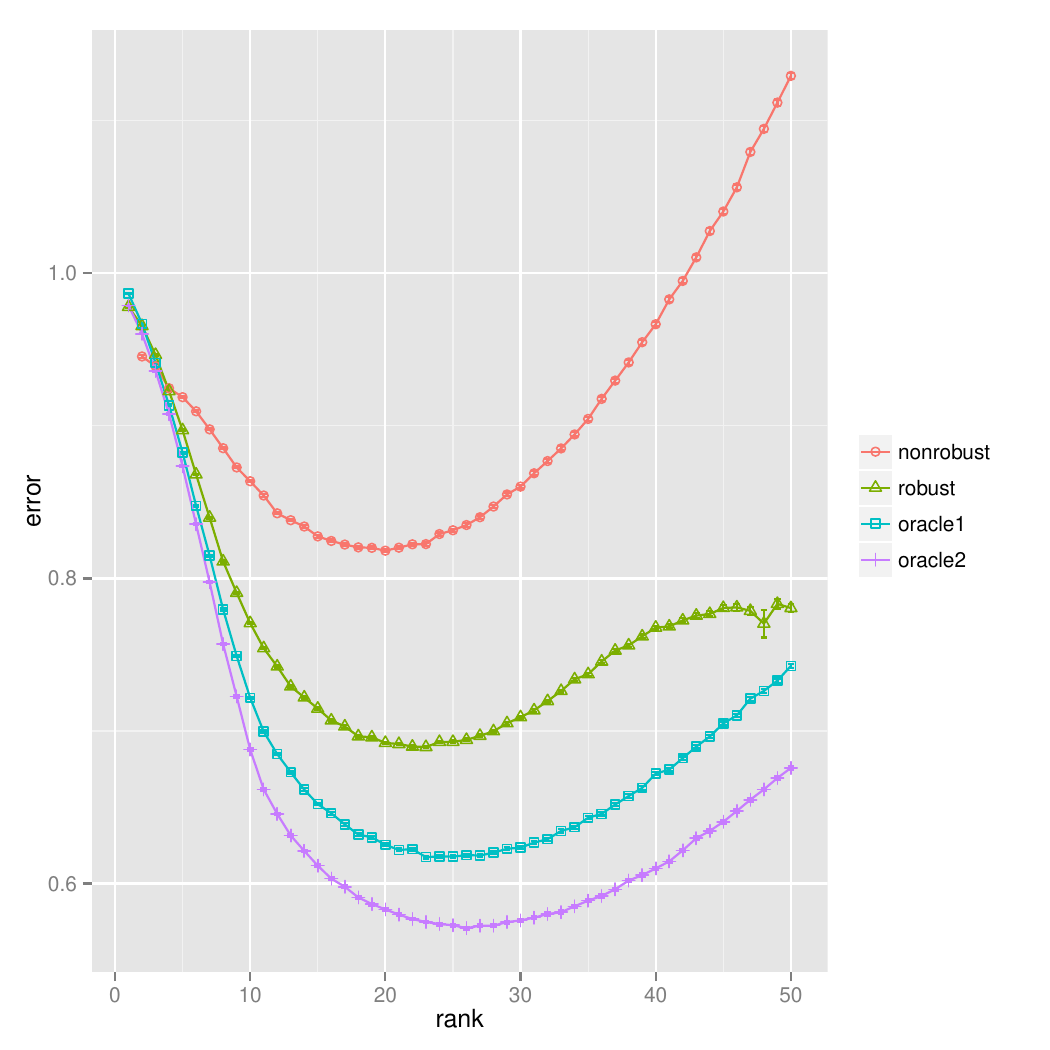}
  \includegraphics[height=9cm]{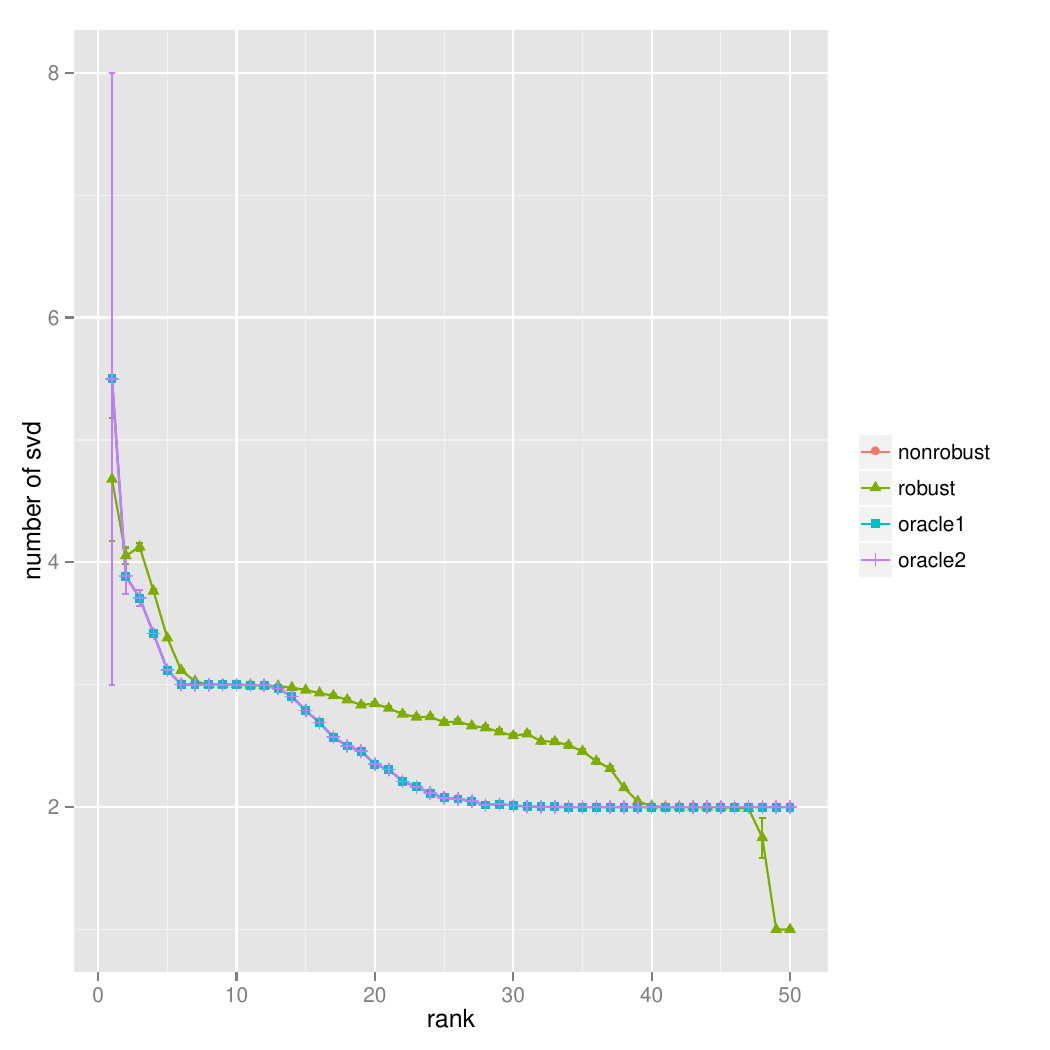}
  \hspace{-0.5cm}
  \includegraphics[height=9cm]{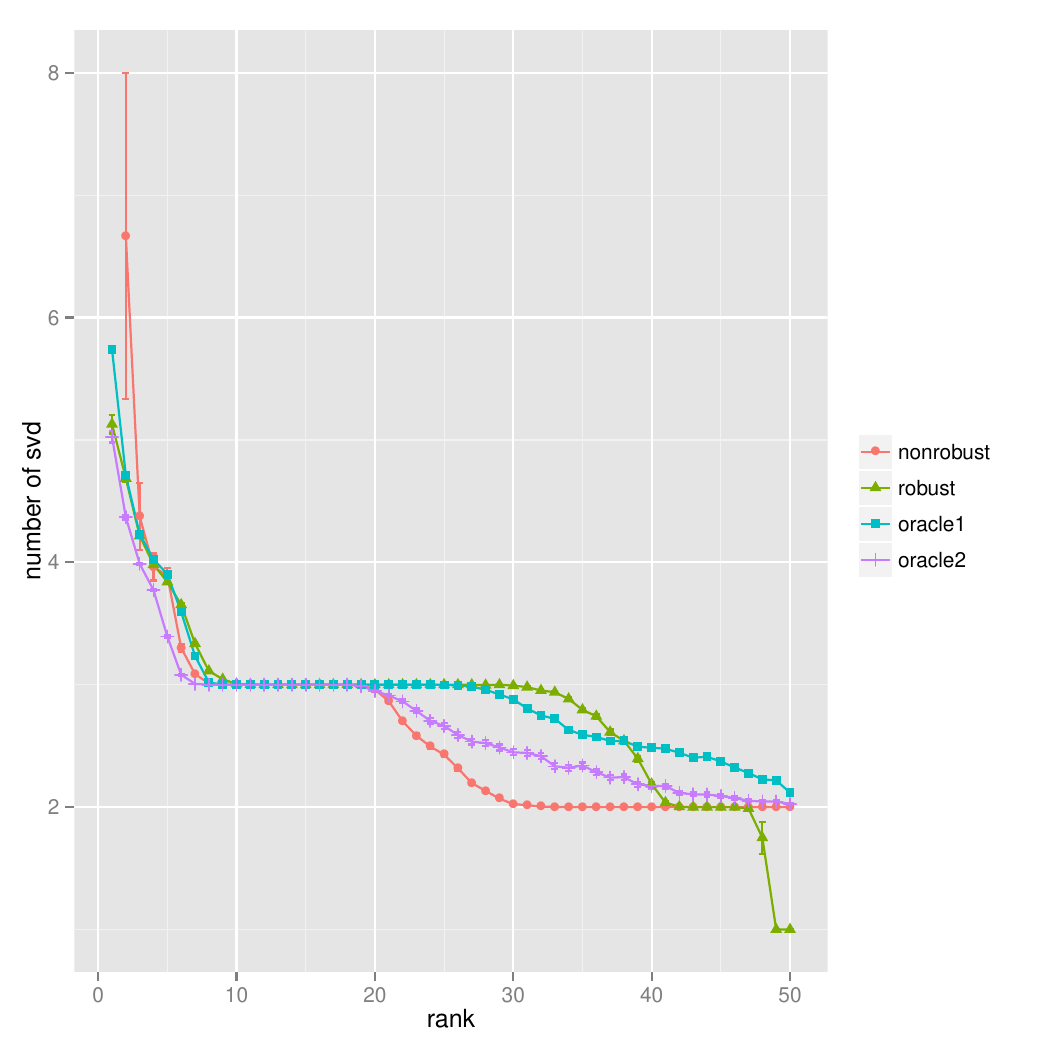}
  \caption{Top: The average test errors with their standard error bands (plus or minus one standard error).
    Bottom: The average number of singular value decompositions used with
    standard error bands (plus or minus one standard error).
    Left: results for the simulation setting: $r=10$, $p=0$ and $q=0.5$.
    Right: results for the simulation setting: $r=10$, $p=0.1$ and $q=0.5$.}\label{fig:sim}
\end{figure}

\subsection{Experiment 2: Image Inpainting}
In this experiment the target matrix is the so-called Lena image that has been used by many authors in the image processing literature.  It consists of $256\times 256$ pixels and is shown in Figure~\ref{fig:lena} (Left).  The simulated data sets were generated via contaminating this Lena image by adding Gaussian noises and/or outliers in the following manner.  First independent Gaussian noise was added to each pixel, where the standard deviation of the noise was set such that the SNR is 3.  Next, 10\% of the pixels were selected as outliers, and to them additional independent Gaussian noises with SNR 3/4 were added.  In terms of selecting missing pixels, two mechanisms were considered.  In the first one 40\% of the pixels were randomly chosen as missing pixels, while in the second mechanism only 10\% were missing but they were clustered together to form patches.  Two typical simulated data sets are shown in Figure~\ref{fig:lena} (Middle).  Note that Theorem~\ref{thm:stable} does not cover the second missing mechanism.  For each missing mechanism, 200 data sets were generated and both \softimpute\ and \robustimpute\ were applied to reconstruct Lena.

The average training and testing errors\footnote{
The solution path (formed by the pre-specified set of $\gamma$'s) may not contain any solution of rank 50, 75, 100 and 125.  Thus, the average errors were computed over those fittings that contained the corresponding fitted ranks.  At most 2\% of these fittings were discarded due to this reason.  
}
of the recovered images of matrix ranks 50, 75, 100 and 125 are reported in Table~\ref{tab:real}.  For both missing mechanisms, \softimpute\ tends to have lower training errors, but larger testing errors when compared to \robustimpute.  In other words, \softimpute\ tends to over-fit the data, and \robustimpute\ seems to provide better results.  Lastly, for visual evaluation, the recovered image of rank 100 using \robustimpute\ is displayed in Figure~\ref{tab:real} (Right).  From this one can see that the proposed \robustimpute\ provided good recoveries under both missing mechanisms.

\begin{table}[htpb]
\centering
\caption{The average training and testing errors for the Lena experiment.}\label{tab:real}
\vspace{0.3cm}
{\small
\begin{tabular}{cc|cccc|cccc}
  \hline\hline
 &&\multicolumn{4}{c}{training error} & \multicolumn{4}{|c}{testing error}\\
 &rank & 50 & 75 & 100 & 125 & 50 & 75 & 100 & 125 \\
  \hline\hline
independent & \softimpute & 0.0499 & 0.0351 & 0.0221 & 0.0113 & 0.0578 & 0.0565 & 0.0581 & 0.0620 \\
missing &\robustimpute & 0.0486 & 0.0371 & 0.0282 & 0.0252 & 0.0546 & 0.0540 & 0.0557 & 0.0571 \\
\hline
clustered & \softimpute & 0.0487 & 0.0386 & 0.0296 & 0.0214 & 0.0756 & 0.0751 & 0.0760 & 0.0781 \\
missing & \robustimpute & 0.0468 & 0.0390 & 0.0321 & 0.0268 & 0.0716 & 0.0714 & 0.0723 & 0.0742 \\
   \hline\hline
\end{tabular}
}
\end{table}

\begin{figure}
  \begin{tabular}{ccc}
 \multirow{-5}{*}{
\includegraphics[height=5cm]{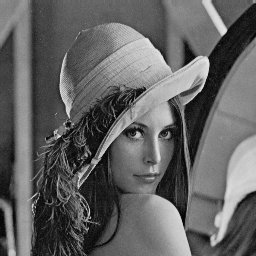}} &
\includegraphics[height=5cm]{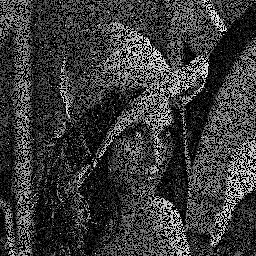} &
\includegraphics[height=5cm]{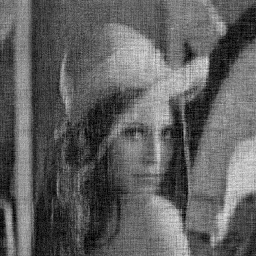}\\  & 
\includegraphics[height=5cm]{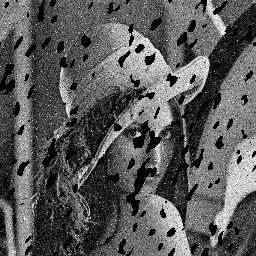} &
\includegraphics[height=5cm]{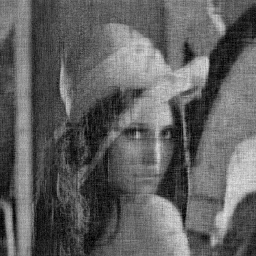}
\end{tabular}
  \caption{Left: the Lena image. Middle: degraded Lena images by the independent missing mechanism (Top) and the clustered missing mechanism (Down). Right: corresponding recovered images of rank 100 via \robustimpute.}\label{fig:lena}
\end{figure}

\subsection{Real data application: Landsat Thematic Mapper}
In this application the target matrix is an image from a Landsat Thematic Mapper data set publicly available at \texttt{http://ternauscover.science.uq.edu.au/}.
This data set contains 149 multiband images of $100\times 100$ pixels, with each image consists of six bands (blue, green and red with three infrared bands).  The scene is centered on the Tumbarumba flux tower on the western slopes of the Snowy Mountains in Australia.  Due to wild fires or related reasons, some pixels are of value zero which can be treated as missing.  Also, due to detector malfunctioning, some isolated pixels have values much higher than the remaining pixels, which can be treated as outliers.  We selected an image band with a high missing rate (27.6\%) to test our procedure.

To evaluate the recovered matrix, the observed pixels were split into training, validation and testing sets
consisting 80\%, 10\% and 10\% of the observed (nonzero) entries respectively.
We used the validation set to tune $\gamma$.
The validation errors are computed in two ways:
mean squared error (MSE) $\sqrt{\sum_{(i,j)\in\mathcal{V}} (X_{ij} - \hat{X}_{ij})^2/|\mathcal{V}|}$ and
mean absolute deviation (MAD) $\mathrm{median}\{|X_{ij} - \hat{X}_{ij}|: (i,j)\in\mathcal{V} \}$,
where $\mathcal{V}$ represents the validation set. Similarly, we compute the
testing errors in terms of MSE and MAD.
Note that the validation and testing sets may contain outliers
and therefore MAD serves as a robust and reliable performance measure.
The corresponding results are shown in Table \ref{tab:real2}.  From this table it can be seen that with the presence of outliers, \robustimpute provided better results.

\begin{table}[htpb]
\centering
\caption{Rank and testing errors of the real data application.}\label{tab:real2}
\vspace{0.3cm}
{\small
\begin{tabular}{c|ccc|ccc}
  \hline\hline
 &\multicolumn{3}{c}{tuning by MSE} & \multicolumn{3}{|c}{tuning by MAD}\\
  & rank & MSE & MAD & rank & MSE & MAD \\
  \hline\hline
 \softimpute & 24 & 45.20 & 31.15 & 21 & 45.23 & 31.15\\
\robustimpute & 24 & 44.63 & 29.00 & 29 & 44.57 & 28.76\\
   \hline\hline
\end{tabular}
}
\end{table}

\section{Concluding remarks}\label{sec:con}
In this paper a classical idea from robust statistics has been brought to the matrix completion problem.  The result is a new matrix completion method that can handle noisy and outlying entries.  This method uses the Huber function to downweigh the effects of outliers.  A new algorithm is developed to solve the corresponding optimization problem.  This algorithm is relatively fast, easy to implement and monotonic convergent.  It can be paired with any existing (non-robust) matrix completion methods to make such methods robust against outliers.  We also developed a specialized version of this algorithm, called \robustimpute.  Its promising empirical performance has been illustrated via numerical experiments.  Lastly, we have shown that the proposed method is stable; that is, with high probability, the error of recovered matrix is bounded by a constant proportional to the noise level.

\section*{Acknowledgment}
The authors are most grateful to the referees and the action editor for their many constructive and useful comments, which led to a much improved version of the paper. The work of Wong was partially supported by the National Science Foundation under Grants DMS-1612985 and DMS-1711952 (subcontract).  The work of Lee was partially supported by the National Science Foundation under Grants DMS-1512945 and DMS-1513484.
\appendix

\section{Technical Details}
\label{app:tech}

\subsection{Proofs of Proposition~\protect\ref{prop:mono}}
\label{app:proofp1}

\begin{proof}

  By rewriting
  \[
    \begin{split}
    \| \PA \bZ^{(k+1)} - \PA\bY^{(k+1)}\|_F^2 &=
    \|\PA\bZ^{(k+1)} - \PA\bY^{(k)}\|_F^2
    + \| \PA\bY^{(k)} - \PA\bY^{(k+1)}\|_F^2\\
    &\quad 2\times \mathrm{trace}\left[\{\PA\bZ^{(k+1)} - \PA\bY^{(k)}\}
    \{\PA\bY^{(k)} - \PA\bY^{(k+1)}\}^\intercal\right],
  \end{split}
  \]
  and using $f(Y^{(k+1)} | Z^{(k+1)}) \le f(Y^{(k)} | Z^{(k+1)})$,
  we have
  \[
    \begin{split}
    &\frac{1}{2}\| \PA\bY^{(k)} - \PA\bY^{(k+1)}\|_F^2 
     + \mathrm{trace}\left[\{\PA\bZ^{(k+1)} - \PA\bY^{(k)}\}
    \{\PA\bY^{(k)} - \PA\bY^{(k+1)}\}^\intercal\right]\\
  & \quad + \gamma\|\bY^{(k+1)}\|_* \le
  \gamma \| \bY^{(k)}\|_*.
\end{split}
  \]
  Thus, by substituting $Z^{(k+1)} = \PA \bY^{(k)} +\frac{1}{2} \rho_c' (\PA\bX-\PA\bY^{(k)})$,
  \begin{equation}
    \begin{split}
    &\frac{1}{2}\| \PA\bY^{(k)} - \PA\bY^{(k+1)}\|_F^2
     + \frac{1}{2}\mathrm{trace}\left[\rho'_c(\PA\bX - \PA\bY^{(k)})
    \{\PA\bY^{(k)} - \PA\bY^{(k+1)}\}^\intercal\right]\\
  & \quad + \gamma\|\bY^{(k+1)}\|_* \le
  \gamma \| \bY^{(k)}\|_*. \label{lem1:eqn1}
\end{split}
  \end{equation}
  Here we abuse the notation slightly so that
  $\rho'_c$ of a matrix simply means the matrix formed by applying
  $\rho'_c$ to its entries.
  Note that for each $(i,j)\in\Oo$, by Taylor's expansion,
  \begin{align*}
    \rho_c(X_{ij}-Y_{ij}^{(k+1)}) &= \rho(X_{ij}- Y_{ij}^{(k)})
    + (Y_{ij}^{(k)} - Y_{ij}^{(k)})\rho'_c(X_{ij} - Y_{ij}^{(k)})
    + \int^{X_{ij} - Y_{ij}^{(k+1)}}_{X_{ij}-Y_{ij}^{(k)}}
    (X_{ij}-Y_{ij}^{(k+1)}- t) \rho_c''(t) dt,
  \end{align*}
  and the last integral term is less than or equal to $(Y_{ij}^{(k)} - Y_{ij}^{(k+1)})^2$
  due to $\rho_c''\le 2$ almost everywhere.
  Thus,
  \[
    \begin{split}
    \sum_{(i,j)\in\Oo} \rho_c( X_{ij}-Y_{ij}^{(k+1)} )
    &\le
    \sum_{(i,j)\in\Oo} \rho_c( X_{ij}-Y_{ij}^{(k)} )\\
    &\quad +
    \mathrm{trace} \left[\rho_c'(\PA\bX - \PA \bY^{(k)})\{ 
      \PA\bY^{(k)} - \PA\bY^{(k+1)} \}^\intercal\right]\\
    &\quad +
    \| \PA\bY^{(k)} - \PA\bY^{(k+1)}\|_F^2.
  \end{split}
  \]
  Now, plugging it into~(\ref{lem1:eqn1}), we have $g(\bY^{(k+1)}) \le g(\bY^{(k)})$.
\end{proof}

\begin{proof}[Alternative proof of Proposition \ref{prop:mono}]
  Similar to the proof of Proposition \ref{prop:equiv} in Section \ref{app:proofp2},
  one can show that
  \begin{equation}
    g(Y) = \min_S \frac{1}{2}\|\POo X - \POo Y - \POo S \|_F^2 + \gamma \|Y\|_* + c\|S\|_1,
    \label{eqn:galter}
  \end{equation}
  where the minimizer is $S(Y)=(1/2) \psi_c(\POo X - \POo Y)$.
  Now, one can show that
  \begin{align*}
    Z^{(k+1)}&=\POo Y^{(k)}+(1/2)\psi_c(\POo X-\POo Y^{(k)})\\
   &= \POo X + (1/2)\psi_c(\POo Y^{(k)} - \POo X)\\
   &= \POo X - (1/2)\psi_c( \POo X - \POo Y^{(k)})\\
   &= \POo X - S(Y^{(k)}).
  \end{align*}
  Now, due to (\ref{eqn:galter}),
  \begin{align*}
    g(Y^{(k+1)}) &=
\min_S \frac{1}{2}\|\POo X - \POo Y^{(k+1)} - \POo S \|_F^2 + \gamma \|Y^{(k+1)}\|_* + c\|S\|_1\\
&\le \frac{1}{2}\|\POo X - \POo Y^{(k+1)} - \POo S(Y^{(k)}) \|_F^2 + \gamma \|Y^{(k+1)}\|_* + c\|S(Y^{(k)})\|_1\\
&= f(Y^{(k+1)} | X -  S(Y^{(k)})) +c\|S(Y^{(k)})\|_1\\
&= f(Y^{(k+1)} |\POo X -  S(Y^{(k)})) +c\|S(Y^{(k)})\|_1\\
&= f(Y^{(k+1)} | Z^{(k+1)}) +c\|S(Y^{(k)})\|_1\\
&\le f(Y^{(k)} | Z^{(k+1)}) +c\|S(Y^{(k)})\|_1\\
&= \frac{1}{2}\|\POo X - \POo Y^{(k)} - \POo S(Y^{(k)}) \|_F^2 + \gamma \|Y^{(k)}\|_* + c\|S(Y^{(k)})\|_1\\
&= g(Y^{(k)}).
  \end{align*}
\end{proof}

\subsection{Proof of Theorem~\protect\ref{thm:global} for Algorithm~\protect\ref{alg:general}}
\label{app:proofalg}
\begin{proof}
This proof closely follows the proofs of Lemma~2.3 and Theorem~3.1 in \citet{Beck-Teboulle09}
by modifying their approximation model to
\begin{align*}
  \zeta(\bY, \tilde{\bY}) &= g_1(\tbY) + \langle \bY - \tbY, \nabla g_1(\tbY)\rangle
  +\frac{1}{2}\| \POo \bY - \POo \tbY\|_F^2 + g_2(\bY)\\
  &= 
  g_1(\tbY) -\frac{1}{2} \langle \bY - \tbY, \psi_c(\POo \bX - \POo \tbY)\rangle
  +\frac{1}{2}\| \POo \bY - \POo \tbY\|_F^2 + g_2(\bY)\\
  &= 
  g_1(\tbY) -\frac{1}{2} \langle \POo\bY - \POo\tbY, \psi_c(\POo \bX - \POo \tbY)\rangle
  +\frac{1}{2}\| \POo \bY - \POo \tbY\|_F^2 + g_2(\bY),
\end{align*}
where $\langle \bX, \bY \rangle = \sum_{i,j} \bX_{ij}\bY_{ij}$.
It can be shown that $\arg\min_\bY \zeta(Y, \tbY)$ is
the same as $\arg\min_\bY f(Y|Z)$, where $Z=\POo \tbY + (1/2)\psi_c(\POo
X - \POo \tbY)$, in Steps 2(a)-(c) of Algorithm \ref{alg:general}.
Let $\Pi(\tbY) =\arg\min_\bY \zeta(\bY, \tbY)$.
Therefore $\bY^{(k+1)}=\Pi(\bY^{(k)})$.
Moreover,
\[
g_1(\bY) \le  g_1(\tbY) + \langle \bY - \tbY, \nabla g_1(\tbY)\rangle
  +\frac{1}{2}\| \POo \bY - \POo \tbY\|_F^2,
\]
for any $\bY$ and $\tbY$.
Therefore, $g(\Pi(\tbY))\le \zeta(\Pi(\tbY), \tbY)$ for any
$\tbY\in\mathbb{R}^{n_1 \times n_2}$.

To proceed, we need a modified version of Lemma~2.3 in \citet{Beck-Teboulle09}.
\begin{lem}\label{lem:global}
  For any $\tbY, \bY\in\mathbb{R}^{n_1\times n_2}$,
  \[
    g(\bY)-g(\PtbY) \ge \frac{1}{2}\| \POo\PtbY - \POo\tbY \|_F^2 + \langle \POo\tbY - \POo\bY,
    \POo\PtbY - \POo\bY\rangle.
  \]
\end{lem}
This lemma is proved as follows.
Since $\PtbY$ is the minimizer of the convex function $\zeta(\cdot , \tbY)$,
there exists a $b(\tbY)\in\partial g_2(\PtbY)$, the subdifferential of $g_2$ at $\PtbY$, such that
$\nabla g_1(\tbY) + \POo \PtbY - \POo \tbY + b(\tbY)=0$.
By the convexity of $g_1$ and $g_2$,
\begin{align*}
  g_1(\bY) &\ge g_1 (\tbY) -\frac{1}{2}\langle \bY - \tbY, \psi_c (\POo X -\POo \tbY)\rangle\\
  g_2(\bY) &\ge g_2(\PtbY) - \langle \bY - \PtbY, b(\tbY)\rangle.
\end{align*}
Therefore,
\begin{equation}
  g(\bY)\ge g_1 (\tbY) -\frac{1}{2}\langle \bY - \tbY, \psi_c (\POo X -\POo \tbY)\rangle
  +g_2(\PtbY) - \langle \bY - \PtbY, b(\tbY)\rangle. \label{lem:global:eq1}
\end{equation}
Since $g(\Pi(\tbY))\le \zeta(\Pi(\tbY), \tbY)$, we have $g(\bY) - g(\PtbY) \ge g(Y) - \zeta(\Pi(\tbY), \tbY)$.
Plugging in (\ref{lem:global:eq1}), the definition of $\zeta$ and the condition for $b$, the conclusion of the lemma follows.

Using Lemma \ref{lem:global} with $\bY=\bY^*$ and $\tbY = \bY^{(k)}$,
we have
\[
2\{  g(\bY^*)- g(\bY^{(k)}) \} \ge \| \POo \bY^* - \POo \bY^{(k+1)}\|_F^2 - \| \POo \bY^* - \POo \bY^{(k)}\|_F^2.
\]
Summing it over $k=0,\dots, m-1$,
\begin{equation}
2\left\{ m g(\bY^*)- \sum^{m-1}_{k=0}g(\bY^{(k)}) \right\} \ge \| \POo \bY^* -
\POo \bY^{(m)}\|_F^2 - \| \POo \bY^* - \POo \bY^{(0)}\|_F^2.\label{eqn:global:eq2}
\end{equation}
Applying Lemma \ref{lem:global} with $\bY=\tbY=\bY^{(k)}$,
\[
2\left\{  g(\bY^{(k)}) - g(\bY^{(k+1)})\right\} \ge \|\POo \bY^{(k+1)} - \POo \bY^{(k)}\|_F^2.
\]
Multiplying it by $k$ and summing over $k=0,\dots, m-1$,
\begin{equation}
  2\left\{-m g(\bY^{(m)}) + \sum^{m-1}_{k=0} g(\bY^{(k+1)}) \right\}
  \ge \sum^{m-1}_{k=0} k \| \POo \bY^{(k+1)}-\POo \bY^{(k)}\|_F^2.\label{eqn:global:eq3}
\end{equation}
Adding (\ref{eqn:global:eq2}) and (\ref{eqn:global:eq3}),
\[
  \begin{split}
2\left\{ g(\bY^*) - g(\bY^{(m)})  \right\} &\ge \| \POo \bY^* -
\POo \bY^{(m)}\|_F^2 - \| \POo \bY^* - \POo \bY^{(0)}\|_F^2\\
&\quad+
\sum^{m-1}_{k=0} k \| \POo \bY^{(k+1)}-\POo \bY^{(k)}\|_F^2.
  \end{split}
\]
Therefore,
\[
g(\bY^{(m)})-g(\bY^*) \le \frac{\|\POo \bY^* - \POo Y^{(0)}\|_F^2}{2m}.
\]
\end{proof}

\subsection{Proof of Proposition~\protect\ref{prop:equiv}}
\label{app:proofp2}

\begin{proof}
  Since both (\ref{eqn:goal}) and (\ref{eqn:goal2}) are convex, we only need to
  consider the sub-gradients.
  The sub-gradient conditions for minimizier of (\ref{eqn:goal}) are given as follows:
  \begin{align}
    0 &\in -\frac{1}{2} \rho'_c(\PA\bX-\PA\bY) + \gamma \partial \|\bY\|_*,
    \label{eqn:lem2:cond1}
  \end{align}
  where $\partial\|\cdot\|_*$ represents the set of subgradients of the nuclear norm.
  The sub-gradient conditions for minimizier of (\ref{eqn:goal2}) are given as follows:
  \begin{align}
    {0} &\in - \PA(\bX-L-S) + \gamma \partial \|L\|_*\label{eqn:lem2:cond21}\\
    {0} &\in - \PA(\bX-L-S) + c \partial \|S\|_1,\label{eqn:lem2:cond22}
  \end{align}
  where
  $\partial\|\cdot\|_1$ represents the set of subgradients of $\|\cdot\|_1$.
  Here (\ref{eqn:lem2:cond22}) implies, for $(i,j)\in\Oo$,
  \begin{equation}
   S_{ij} =\begin{cases}
     X_{ij} - L_{ij} - c, & X_{ij} - L_{ij}>c\\
     0, & |X_{ij} -L_{ij} | \le c\\
     X_{ij} - L_{ij} + c, & X_{ij}- L_{ij} <-c \label{eqn:lem2:eqn1}
   \end{cases}
 \end{equation}
  and $S_{ij}=0$ for $(i,j)\in\Oo^\perp$.
  Note, for $(i,j)\in \Oo$, $X_{ij}-L_{ij}-S_{ij}=\rho_c'(X_{ij}-L_{ij})/2$.
  Plugging it into (\ref{eqn:lem2:cond21}),
  we have (\ref{eqn:lem2:cond1}) and thus this proves the proposition.
\end{proof}

\subsection{Proof of Theorem~\protect\ref{thm:stable}}
\label{app:prooft1}
{To prove Theorem \ref{thm:stable}, we first show three lemmas and one proposition.

\begin{lem}[Modified Lemma A.2 in \citep{Candes-Li-Ma11}]\label{lem:stable4}
  Assume that for any matrix $Q$, $\|\PT \PGc Q\|_F\le n \|\PTp\PGc Q\|_F$.
  Suppose there is a pair $(W,F)$ obeying
  \begin{align}
    \begin{cases}
    \PT W =0,\quad \|W\| < 1/2,\\
    \PGc F=0, \quad \|F\|_\infty< 1/2,\\
    UV^\intercal + W + \PT D = \lambda (\mathrm{sgn}(S_0') + F)\quad \mathrm{with}\quad
    \| \PT D\|_F \le n^{-2}.
    \label{eqn:dual}
  \end{cases}
  \end{align}
  Then for any perturbation $H=(H_L, H_S)$ satisfying $\POo H_L+ \POo H_S=0$,
  \[
    \| M_0 -H\|_\Diamond \ge \| M_0\|_\Diamond +
    \left(\frac{1}{2} - \frac{1}{n}\right)\|\PTp H_L\|_*
    + \left( \frac{\lambda}{2} - \frac{n+1}{n^2}\right) \| \PG H_L\|_1.
  \]
  \end{lem}

The proof of this lemma can be found in \citet{Candes-Li-Ma11}.
To procced, we write $\|M\|_{F,\lambda}^2 = \|L\|_F^2 + \lambda^2\|S\|_F^2$
for any pair of matrices $M=(L,S)$.

\begin{lem}\label{lem:stable3}
  Let $M=(M_L, M_S)$ be any pair of matrices.
  Suppose $\|\POo \PT M_L \|_F^2 \ge \pz\| \PT M_L\|_F^2/2$ and
  $\|\PT \PO\|^2 \le \pz/8$. Then
  \[
    \| \PP (\PT \times \PO) M \|_{F,\lambda}^2 \ge \frac{ (1+\lambda^2)\pz}{16} \| (\PT \times \PO) M\|_F^2.
  \]
\end{lem}

\begin{proof}[Proof of Lemma \ref{lem:stable3}]
  Note that for any $M'=(M'_L,M_S')$,
  \[
    \PP M' = \left(\frac{\POo (M_L' + M_S')}{2}, \frac{\POo (M_L' + M_S')}{2}\right).
  \]
  Thus
  \begin{align*}
    \| \PP (\PT \times \PO) M \|_{F,\lambda}^2 &= \frac{1+\lambda^2}{4} \| \POo (\PT M_L + \PO M_S) \|_F^2\\
    &= \frac{1+\lambda^2}{4} \left( \| \POo \PT M_L \|_F^2 + \| \PO M_S\|_F^2 
    + 2\langle \POo \PT M_L , \PO M_S\rangle\right),
  \end{align*}
  where the last equality is due to $\Omega \subset \Oo$. By
  $\|\PT \PO\|^2 \le \pz/8$,
  \begin{align*}
    \langle \POo \PT M_L, \PO M_S \rangle &= \langle \PT M_L, \PO M_S\rangle\\
    &= \langle \PT M_L, (\PT\PO) \PO M_S\rangle\\
    &\ge - \| \PT \PO\| \| \PT M_L \|_F \|\PO M_S\|_F\\
    &\ge - \frac{\sqrt{\pz}}{2\sqrt{2}}\|\PT M_L \|_F \| \PO M_S\|_F.
  \end{align*}
  Combining with $\|\POo \PT M_L \|_F^2 \ge \pz \|\PT M_L\|_F^2/2$, we have
  \[
    \| \PP (\PT\times \PO) M\|_{F,\lambda}^2
    \ge \frac{1+\lambda^2}{4} \left( \frac{\pz}{2} \| \PT M_L\|_F^2 + \|\PO M_S\|_F^2
    - \sqrt{\frac{{\pz}}{{2}}}\| \PT M_L \|_F \|\PO M_S\|_F \right).
  \]
  As $2(x^2 + y^2 -xy) \ge x^2 + y^2$ for $x, y \ge 0$,
  \[
    \| \PP (\PT\times \PO) M\|_{F,\lambda}^2
    \ge \frac{1+\lambda^2}{8}\left(  \frac{\pz}{2}\|\PT M_L\|_F^2 + \| \PO M_S\|_F^2 \right)
    \ge \frac{(1+\lambda^2)\pz}{16}  \| (\PT \times \PO) M \|_F^2.
  \]
\end{proof}

\begin{lem}\label{lem:stable6}
Let $M=(M_L, M_S)$ be any pair of matrices. Then $\|\PP M\|_{F,\lambda}^2\le \|M\|_{F,\lambda}^2/2$.
\end{lem}

\begin{proof}[Proof of Lemma \ref{lem:stable6}]
  Write $\MP=(\MP_L, \MP_S)= \PP M$.
  Since $\|\MP_L\|_F^2=\|\MP_S\|_F^2$,
  \begin{align*}
    \|\PP M\|_{F,\lambda}^2 &= \|\MP_L\|_F^2 + \lambda^2 \| \MP_S\|_F^2\\
    &= \frac{1}{2}(\|\MP_L\|_F^2 + \|\MP_S\|_F^2)
    +\frac{\lambda^2}{2}(\|\MP_L\|_F^2+ \|\MP_S\|_F^2)\\
    &=\frac{1}{2}\|\MP\|_F^2 + \frac{\lambda^2}{2}\|\MP\|_F^2\\
    &\le \frac{1}{2}\|M\|_F^2 + \frac{\lambda^2}{2}\|M\|_F^2=\frac{1}{2}\|M\|_{F,\lambda}^2.
  \end{align*}
\end{proof}

\begin{prop}\label{prop:stable}
  Assume that for any matrix $Q$, $\|\PT \PGc Q\|_F\le n \|\PTp\PGc Q\|_F$ and
  $\|\POo \PT Q\|_F\ge \pz \|\PT Q\|_F/2$. Further suppose
  $4/n<\lambda \le 1$, $n\ge 3$, $p_0>0$,
  $\| \PT \PO\|^2 \le \pz/8$ and
  that there exists a pair $(W,F)$ obeying (\ref{eqn:dual}).
  Then the solution $\hat{M}=(\hat{L}, \hat{S})$ to (\ref{eqn:goal2})
  satisfies
    \[
    \| \hat{M}-M_0 \|_{F,\lambda} \le \left[\sqrt{1+\lambda^2} +
    4\left(1+\sqrt{\frac{8}{{\pz}}}\right)  (\sqrt{n}+n\lambda\sqrt{p_0})\right]\delta.
  \]
  where $M_0=(L_0, S_0')$ such that $\| \POo X - \POo (L_0 +
  S_0))\|_F^2 \le \delta$ and $S_0'=\POo S_0$.
  Further, if $\lambda=1/\sqrt{n\pz}$ (which implies $1/n<p_0<n/16$),
  we obtain
  \[
    \|\hat{L}-L\|_F \le \left\{2 + 8\sqrt{n}\left(1+\sqrt{\frac{8}{{\pz}}}\right) \right\} \delta
    \qquad\mbox{and}\qquad 
    \|\hat{S}-S_0'\|_F \le \left\{2 + 8\sqrt{n}\left(1+\sqrt{\frac{8}{{\pz}}}\right) \right\}\sqrt{np_0} \delta.
  \]
\end{prop}

\begin{proof}[Proof of Proposition \ref{prop:stable}]
  Write $\hat{M} = M_0 + H$, where $H=(H_L, H_S)$,
  and $\HP=(\HP_L, \HP_S)= \PP H$ and $\HPp=(\HPp_L, \HPp_S) = \PPp H$.
  We want to bound
  \begin{align}
    \|H\|_{F,\lambda} &= \|\HP+\HPp\|_{F,\lambda}\nonumber\\
    &\le \|\HP\|_{F,\lambda} + \|\HPp\|_{F,\lambda}\nonumber\\
    &\le \|\HP\|_{F,\lambda} + \| (\PTp \times \PG) \HPp\|_{F,\lambda} + \| (\PT \times \PGc) \HPp\|_{F,\lambda}.\label{eqn:prop:stable:1}
  \end{align}
  We start with the first term of (\ref{eqn:prop:stable:1}).
  Since $\HP_L=\HP_S=(1/2)\POo (H_L+H_S)$,
  \begin{align*}
    \|\HP\|_{F,\lambda} &= \frac{\sqrt{1+\lambda^2}}{2}\| \POo (H_L + H_S)\|_F\\
    &= \frac{\sqrt{1+\lambda^2}}{2}\| \POo (\hat{L}+\hat{S} - L_0 - S_0')\|_F\\
    &\le \frac{\sqrt{1+\lambda^2}}{2} \left(\| \POo (\hat{L}+\hat{S} -X)\|_F +
      \| \POo(L_0 + S_0' - X)\|_F\right)\\
    &\le \delta\sqrt{1+\lambda^2},
  \end{align*}
  where the last inequality is due to
  the fact that both $M_0$ and $\hat{M}$ are feasible.

  Then we focus on the second term of (\ref{eqn:prop:stable:1}).
  First, we have
  \[
    \| M_0 \|_\Diamond \ge \| \hat{M}\|_\Diamond = \| M_0 + H\|_\Diamond
    \ge \| M_0 + \HPp\|_\Diamond - \| \HP\|_\Diamond .
  \]
  By Lemma \ref{lem:stable4},
  \[
    \|M_0 + \HPp\|_\Diamond \ge \| M_0\|_\Diamond
    + a(n) \| \PTp \HPp_L\|_* + b(n, \lambda) \| \PG \HPp_L\|_1,
  \]
  where
  \[
    a(n) = \frac{1}{2} - \frac{1}{n} \quad\mbox{and}\quad
      b(n, \lambda) =  \frac{\lambda}{2} - \frac{n+1}{n^2}.
  \]
  Now, combining the above inequalities,
  \begin{equation}
    \| \HP\|_\Diamond \ge
     a(n) \| \PTp \HPp_L\|_* + b(n, \lambda) \| \PG \HPp_L\|_1.
     \label{eqn:prop:stable:2}
   \end{equation}
  By the assumption that $\lambda > 4/n$ and $n\ge 3$,
  \begin{align*}
    a(n) =\frac{1}{2}- \frac{1}{n} > 0 \quad\mbox{and}\quad
    b(n, \lambda) =  \frac{\lambda}{2} - \frac{n+1}{n^2} > \frac{2}{n}-\frac{1}{n}-\frac{1}{n^2}=\frac{1}{n}-\frac{1}{n^2}>0.
  \end{align*}
   Therefore (\ref{eqn:prop:stable:2}) implies $\| \HP\|_\Diamond \ge a(n) \| \PTp \HPp_L\|_*$ and 
   $\| \HP\|_\Diamond \ge b(n, \lambda) \| \PG \HPp_L\|_1$.

   Now, we are ready to establish a bound for the second term of
   (\ref{eqn:prop:stable:1}).
   \begin{align*}
     \| (\PTp \times \PG) \HPp\|_{F,\lambda} &\le
     \| \PTp \HPp_L\|_F + \lambda\| \PG \HPp_S\|_F\\
     &\le \| \PTp \HPp_L\|_* +\lambda \| \PG \HPp_S\|_1\\
     &\le \left\{\frac{1}{a(n)} + \frac{\lambda}{b(n,\lambda)}\right\} \|\HP\|_\Diamond\\
     &\le 4 (\|\HP_L\|_* + \lambda \|\HP_S\|_1).
   \end{align*}

  As for the third term of (\ref{eqn:prop:stable:1}), we apply
  Lemma~\ref{lem:stable3} and the bound of the second term in (\ref{eqn:prop:stable:1}).
  As $\PP\HPp=0$, $\PP(\PT \times \PGc) \HPp + \PP (\PTp \times \PG) \HPp=0$. Therefore,
  due to Lemma \ref{lem:stable6},
  \begin{align*}
    \| \PP(\PT \times \PGc) \HPp\|_{F,\lambda} = \| \PP (\PTp \times \PG) \HPp\|_{F,\lambda}
    \le \frac{1}{\sqrt{2}} \| (\PTp \times \PG) \HPp\|_{F,\lambda}.
  \end{align*}
  As $\POoc H_S$ does not affect the feasibility of $M+H$ and
  $H$ is chosen such that $\|M+H\|_\Diamond$ is minimized, thus $\POoc \HPp_S=\POoc H_S = 0$
  which implies
  $(\PT \times \PGc) \HPp= (\PT \times \PO) \HPp$.
  Thus, by Lemma \ref{lem:stable3},
  \begin{align*}
    \| (\PT \times \PGc) \HPp\|_F &\le \sqrt{\frac{8}{{(1+\lambda^2)\pz}}} \| (\PTp \times \PG) \HPp\|_{F,\lambda}\le  \sqrt{\frac{8}{{\pz}}} \| (\PTp \times \PG) \HPp\|_{F,\lambda}.
  \end{align*}
   And $\| (\PT \times \PGc) \HPp\|_{F,\lambda}\le \| (\PT \times \PGc) \HPp\|_{F}$ as $\lambda\le 1$.

  Collecting all the above bounds for the three terms, we derive the bound for $\|H\|_{F,\lambda}$:
  \[
    \| H \|_{F,\lambda} \le \delta\sqrt{1+\lambda^2} +
    4\left(1+\sqrt{\frac{8}{{\pz}}}\right) (\|\HP_L\|_* + \lambda \|\HP_S\|_1).
  \]
  Finally, $\|\HP_L\|_*\le \sqrt{n}\|\HP_L\|_F$, $\|\HP_S\|_1=\sqrt{p_0n^2}\|\HP_S\|_F$
  (since $\HP_S$ is supported on $\Oo$)
  and $\|\HP_L\|_F=\|\HP_S\|=\|\POo (H_L + H_S)\|_F/2 \le \delta$. Therefore,
  \[
    \| H \|_{F,\lambda} \le \delta\left[\sqrt{1+\lambda^2} +
    4\left(1+\sqrt{\frac{8}{{\pz}}}\right)  (\sqrt{n}+n\lambda\sqrt{p_0})\right].
  \]
  Assume that $\lambda=1/\sqrt{n\pz}$.
  First we note that, due to $\lambda>4/n$, this condition imposes a reasonable
  coverage of $p_0$: $1/n<p_0<n/16$.
  Now we focus on simplifying the bound for $\|H\|_{F,\lambda}$.
  \begin{align*}
    &\sqrt{1+\lambda^2} +
    4\left(1+\sqrt{\frac{8}{{\pz}}}\right) (\sqrt{n}+n\lambda\sqrt{p_0})
      \le 2 + 8\sqrt{n}\left(1+\sqrt{\frac{8}{{\pz}}}\right). 
  \end{align*}
  This implies
  \[
    \|H_L\|_F \le \left\{2 + 8\sqrt{n}\left(1+\sqrt{\frac{8}{{\pz}}}\right) \right\} \delta
    \qquad\mbox{and}\qquad 
    \|H_S\|_F \le \left\{2 + 8\sqrt{n}\left(1+\sqrt{\frac{8}{{\pz}}}\right) \right\}\sqrt{np_0} \delta.
  \]
\end{proof}

To prove Theorem \ref{thm:stable}, we establish one additional lemma.

\begin{lem}\label{lem:stable1}
  Suppose $\| \PT - \pz^{-1} \PT \POo \PT\| \le 1/2$.  Then for any matrix $Q$,
  \[
    \| \POo \PT Q \|_F^2 \ge \frac{\pz}{2} \| \PT Q\|_F^2.
  \]
\end{lem}

\begin{proof}[Proof of Lemma \ref{lem:stable1}]
  By the assumptions, for any matrix $Q$,
  \begin{align*}
    \| \POo \PT Q \|_F^2 &= \langle \POo \PT Q, \POo \PT Q \rangle\\
    &= \langle \PT Q, \PT \POo \PT Q \rangle\\
    &= \pz  \langle \PT Q, \pz^{-1} \PT \POo \PT Q \rangle\\
    &= \pz \left[ \| \PT {Q}\|_F^2 + \langle \PT Q, (\pz^{-1} \PT\POo\PT - \PT) Q \right]\\
    &\le \pz \left(\|\PT Q\|_F^2 - \frac{1}{2} \| \PT Q \|_F^2 \right)\\
    &= \frac{\pz}{2}\|\PT Q \|_F^2.
  \end{align*}
\end{proof}

\begin{proof}[Proof of Theorem \ref{thm:stable}]
Recall that we write that an event occurs with high probability if it holds
with probability at least $1-\bigO(n^{-10})$.
Due to the asymptotic nature of Theorem \ref{thm:stable},
we only require the conditions of Proposition \ref{prop:stable}
to hold asymptotically with large probability.
  By Lemma A.3 of \citet{Candes-Li-Ma11},
  $\|\PT\PQ\|_F \le n\|\PTp\PQ\|_F$ for all $Q$, with high probability.
  By Lemma \ref{lem:stable1} and Theorem 2.6 of \citet{Candes-Li-Ma11} \citep[see also][Theorem~4.1]{Candes-Recht09},
  $\| \POo \PT Q \|_F^2 \ge \frac{\pz}{2} \| \PT Q\|_F^2$ for all $Q$, with high
  probability.
  Further, by \citet{Candes-Recht09}, 
  $\|\PT\PO\|^2 \le \pz/8$ occurs with high probability.
  \citet[][pp.~33-35]{Candes-Li-Ma11} show that there exist dual
  certificates $(W,F)$ obeying (\ref{eqn:dual}) with high probability.
  For sufficiently large $n$, the conditions of $\lambda$ and $p_0$ in Proposition~\ref{prop:stable} are fulfilled. Therefore,
Theorem~\ref{thm:stable} follows from Proposition~\ref{prop:stable}.
\end{proof}
}

\bibliographystyle{chicago}
\bibliography{raywongref}

\begin{thebibliography}{}

\bibitem[\protect\citeauthoryear{Beck and Teboulle}{Beck and
  Teboulle}{2009}]{Beck-Teboulle09}
Beck, A. and M.~Teboulle (2009).
\newblock A fast iterative shrinkage-thresholding algorithm for linear inverse
  problems.
\newblock {\em SIAM Journal on Imaging Sciences\/}~{\em 2\/}(1), 183--202.

\bibitem[\protect\citeauthoryear{Bennett and Lanning}{Bennett and
  Lanning}{2007}]{Bennett-Lanning07}
Bennett, J. and S.~Lanning (2007).
\newblock The netflix prize.
\newblock In {\em Proceedings of KDD cup and workshop}, Volume 2007, pp.\ ~35.

\bibitem[\protect\citeauthoryear{Cai, Cand{\'e}s, and Shen}{Cai
  et~al.}{2010}]{Cai-Candes-Shen10}
Cai, J.-F., E.~J. Cand{\'e}s, and Z.~Shen (2010).
\newblock A singular value thresholding algorithm for matrix completion.
\newblock {\em SIAM Journal on Optimization\/}~{\em 20\/}(4), 1956--1982.

\bibitem[\protect\citeauthoryear{Cand{\`e}s, Li, Ma, and Wright}{Cand{\`e}s
  et~al.}{2011}]{Candes-Li-Ma11}
Cand{\`e}s, E.~J., X.~Li, Y.~Ma, and J.~Wright (2011).
\newblock Robust principal component analysis?
\newblock {\em Journal of the ACM (JACM)\/}~{\em 58\/}(3), Article 11.

\bibitem[\protect\citeauthoryear{Cand{\`e}s and Plan}{Cand{\`e}s and
  Plan}{2010}]{Candes-Plan10}
Cand{\`e}s, E.~J. and Y.~Plan (2010).
\newblock Matrix completion with noise.
\newblock {\em Proceedings of the IEEE\/}~{\em 98\/}(6), 925--936.

\bibitem[\protect\citeauthoryear{Cand{\`e}s and Recht}{Cand{\`e}s and
  Recht}{2009}]{Candes-Recht09}
Cand{\`e}s, E.~J. and B.~Recht (2009).
\newblock Exact matrix completion via convex optimization.
\newblock {\em Foundations of Computational mathematics\/}~{\em 9\/}(6),
  717--772.

\bibitem[\protect\citeauthoryear{Chandrasekaran, Sanghavi, Parrilo, and
  Willsky}{Chandrasekaran et~al.}{2011}]{Chandrasekaran-Sanghavi-Parrilo11}
Chandrasekaran, V., S.~Sanghavi, P.~A. Parrilo, and A.~S. Willsky (2011).
\newblock Rank-sparsity incoherence for matrix decomposition.
\newblock {\em SIAM Journal on Optimization\/}~{\em 21\/}(2), 572--596.

\bibitem[\protect\citeauthoryear{Chen, Xu, Caramanis, and Sanghavi}{Chen
  et~al.}{2011}]{Chen-Xu-Caramanis11}
Chen, Y., H.~Xu, C.~Caramanis, and S.~Sanghavi (2011).
\newblock Robust matrix completion and corrupted columns.
\newblock In {\em Proceedings of the 28th International Conference on Machine
  Learning (ICML-11)}, pp.\  873--880.

\bibitem[\protect\citeauthoryear{Gross}{Gross}{2011}]{Gross11}
Gross, D. (2011).
\newblock Recovering low-rank matrices from few coefficients in any basis.
\newblock {\em IEEE Transactions on Information Theory\/}~{\em 57\/}(3),
  1548--1566.

\bibitem[\protect\citeauthoryear{Hastie, Mazumder, Lee, and Zadeh}{Hastie
  et~al.}{2014}]{Hastie-et-al14}
Hastie, T., R.~Mazumder, J.~Lee, and R.~Zadeh (2014).
\newblock Matrix completion and low-rank {SVD} via fast alternating least
  squares.
\newblock Unpublished manuscript.

\bibitem[\protect\citeauthoryear{Huber and Ronchetti}{Huber and
  Ronchetti}{2011}]{Huber-Ronchetti11}
Huber, P.~J. and E.~M. Ronchetti (2011).
\newblock {\em Robust Statistics\/} (Second ed.), Volume 693.
\newblock New Jersey: John Wiley \& Sons.

\bibitem[\protect\citeauthoryear{Hunter and Lange}{Hunter and
  Lange}{2004}]{Hunter-Lange04}
Hunter, D.~R. and K.~Lange (2004).
\newblock A tutorial on mm algorithms.
\newblock {\em The American Statistician\/}~{\em 58\/}(1), 30--37.

\bibitem[\protect\citeauthoryear{Karhunen}{Karhunen}{2011}]{Karhunen11}
Karhunen, J. (2011).
\newblock Robust pca methods for complete and missing data.
\newblock {\em Neural Network World\/}~{\em 21\/}(5), 357.

\bibitem[\protect\citeauthoryear{Keshavan, Montanari, and Oh}{Keshavan
  et~al.}{2010a}]{Keshavan-Montanari-Oh10a}
Keshavan, R.~H., A.~Montanari, and S.~Oh (2010a).
\newblock Matrix completion from a few entries.
\newblock {\em IEEE Transactions on Information Theory\/}~{\em 56\/}(6),
  2980--2998.

\bibitem[\protect\citeauthoryear{Keshavan, Montanari, and Oh}{Keshavan
  et~al.}{2010b}]{Keshavan-Montanari-Oh10}
Keshavan, R.~H., A.~Montanari, and S.~Oh (2010b).
\newblock Matrix completion from noisy entries.
\newblock {\em Journal of Machine Learning Research\/}~{\em 11\/}(1),
  2057--2078.

\bibitem[\protect\citeauthoryear{Koltchinskii, Lounici, Tsybakov,
  et~al.}{Koltchinskii et~al.}{2011}]{Koltchinskii-Lounici-Tsybakov11}
Koltchinskii, V., K.~Lounici, A.~B. Tsybakov, et~al. (2011).
\newblock Nuclear-norm penalization and optimal rates for noisy low-rank matrix
  completion.
\newblock {\em The Annals of Statistics\/}~{\em 39\/}(5), 2302--2329.

\bibitem[\protect\citeauthoryear{Lange}{Lange}{2010}]{Lange10}
Lange, K. (2010).
\newblock {\em Numerical Analysis for Statisticians}.
\newblock New York: Springer.

\bibitem[\protect\citeauthoryear{Lange, Hunter, and Yang}{Lange
  et~al.}{2000}]{Lange-Hunter-Yang00}
Lange, K., D.~R. Hunter, and I.~Yang (2000).
\newblock Optimization transfer using surrogate objective functions.
\newblock {\em Journal of Computational and Graphical Statistics\/}~{\em
  9\/}(1), 1--20.

\bibitem[\protect\citeauthoryear{Luttinen, Ilin, and Karhunen}{Luttinen
  et~al.}{2012}]{Luttinen-Ilin-Karhunen12}
Luttinen, J., A.~Ilin, and J.~Karhunen (2012).
\newblock Bayesian robust pca of incomplete data.
\newblock {\em Neural processing letters\/}~{\em 36\/}(2), 189--202.

\bibitem[\protect\citeauthoryear{Ma, Goldfarb, and Chen}{Ma
  et~al.}{2011}]{Ma-Goldfarb-Chen11}
Ma, S., D.~Goldfarb, and L.~Chen (2011).
\newblock Fixed point and bregman iterative methods for matrix rank
  minimization.
\newblock {\em Mathematical Programming\/}~{\em 128\/}(1-2), 321--353.

\bibitem[\protect\citeauthoryear{Marjanovic and Solo}{Marjanovic and
  Solo}{2012}]{Marjanovic-Solo12}
Marjanovic, G. and V.~Solo (2012).
\newblock On $l_q$ optimization and matrix completion.
\newblock {\em IEEE Transactions on Signal Processing\/}~{\em 60\/}(11),
  5714--5724.

\bibitem[\protect\citeauthoryear{Mazumder, Hastie, and Tibshirani}{Mazumder
  et~al.}{2010}]{Mazumder-Hastie-Tibshirani10}
Mazumder, R., T.~Hastie, and R.~Tibshirani (2010).
\newblock Spectral regularization algorithms for learning large incomplete
  matrices.
\newblock {\em The Journal of Machine Learning Research\/}~{\em 11},
  2287--2322.

\bibitem[\protect\citeauthoryear{Montanari and Oh}{Montanari and
  Oh}{2010}]{Montanari-Oh10}
Montanari, A. and S.~Oh (2010).
\newblock On positioning via distributed matrix completion.
\newblock In {\em Sensor Array and Multichannel Signal Processing Workshop
  (SAM), 2010 IEEE}, pp.\  197--200.

\bibitem[\protect\citeauthoryear{Nesterov}{Nesterov}{2007}]{Nesterov07}
Nesterov, Y. (2007).
\newblock Gradient methods for minimizing composite objective function.
\newblock Technical report, CORE.

\bibitem[\protect\citeauthoryear{Oh, Nychka, and Lee}{Oh
  et~al.}{2007}]{Oh-Nychka-Lee07}
Oh, H.-S., D.~W. Nychka, and T.~C.~M. Lee (2007).
\newblock The role of pseudo data for robust smoothing with application to
  wavelet regression.
\newblock {\em Biometrika\/}~{\em 94\/}(4), 893--904.

\bibitem[\protect\citeauthoryear{Recht}{Recht}{2011}]{Recht11}
Recht, B. (2011).
\newblock A simpler approach to matrix completion.
\newblock {\em The Journal of Machine Learning Research\/}~{\em 12},
  3413--3430.

\bibitem[\protect\citeauthoryear{Rennie and Srebro}{Rennie and
  Srebro}{2005}]{Rennie-Srebro05}
Rennie, J.~D. and N.~Srebro (2005).
\newblock Fast maximum margin matrix factorization for collaborative
  prediction.
\newblock In {\em Proceedings of the 22nd international conference on Machine
  learning}, pp.\  713--719.

\bibitem[\protect\citeauthoryear{She and Owen}{She and Owen}{2011}]{She-Owen11}
She, Y. and A.~B. Owen (2011).
\newblock Outlier detection using nonconvex penalized regression.
\newblock {\em Journal of the American Statistical Association\/}~{\em
  106\/}(494), 626--639.

\bibitem[\protect\citeauthoryear{Srebro and Jaakkola}{Srebro and
  Jaakkola}{2003}]{Srebro-Jaakkola03}
Srebro, N. and T.~Jaakkola (2003).
\newblock Weighted low-rank approximations.
\newblock In {\em ICML}, Volume~3, pp.\  720--727.

\bibitem[\protect\citeauthoryear{Weinberger and Saul}{Weinberger and
  Saul}{2006}]{Weinberger-Saul06}
Weinberger, K.~Q. and L.~K. Saul (2006).
\newblock Unsupervised learning of image manifolds by semidefinite programming.
\newblock {\em International Journal of Computer Vision\/}~{\em 70\/}(1),
  77--90.

\bibitem[\protect\citeauthoryear{Zhou, Li, Wright, Candes, and Ma}{Zhou
  et~al.}{2010}]{Zhou-Li-Wright10}
Zhou, Z., X.~Li, J.~Wright, E.~Candes, and Y.~Ma (2010).
\newblock Stable principal component pursuit.
\newblock In {\em 2010 IEEE International Symposium on Information Theory
  Proceedings (ISIT)}, pp.\  1518--1522.

\end{thebibliography}
\end{document}